\documentclass[10pt,journal,twocolumn,twoside]{IEEEtran} 
\pdfoutput=1 
\usepackage[utf8]{inputenc}
\usepackage{mydefinitions} 
\usepackage{irtdefinitions} 
\usepackage{cite}
\usepackage{amsmath,amssymb,amsfonts}
\usepackage{algorithm,algpseudocode}
\usepackage{color}
\usepackage{multirow}
\usepackage{booktabs}
\usepackage{mathtools}
\usepackage{fancyhdr}
\usepackage{xcolor}

\usepackage{hyperref}
\hypersetup{
	pdftitle={Motion Planning using Reactive Circular Fields: A 2D Analysis of Collision Avoidance and Goal Convergence},
	pdfauthor={Marvin Becker},
	colorlinks,
	linkcolor={red!50!black},
	citecolor={blue!50!black},
	urlcolor={blue!80!black}
}

\def\doi{10.1109/TAC.2023.3303168}
\fancyfootoffset{-1em}
\fancypagestyle{copyright}{
	\chead{\footnotesize This article has been accepted for publication in IEEE Transactions on Automatic Control. This is the author's version which has not been fully edited and content may change prior to final publication. Citation information: DOI \doi
    }
}

\begin{document}
\title{Motion Planning using Reactive Circular Fields:\\ A 2D Analysis of Collision Avoidance and Goal Convergence}
\author{Marvin Becker, Johannes Köhler, Sami Haddadin and Matthias A. Müller
    \thanks{This work was supported in part by the Region Hannover in the project \emph{roboterfabrik}. \emph{(Corresponding author: Marvin Becker.)}}
    \thanks{M. Becker and M. A. Müller are with the Institute of Automatic Control, Leibniz University Hannover, Germany, (e-mail:becker@irt.uni-hannover.de; mueller@irt.uni-hannover.de)}
    \thanks{J. Köhler is with the Institute for Dynamic Systems and Control, ETH Zürich, Zürich CH-8092, Switzerland, (e-mail: jkoehle@ethz.ch).}
    \thanks{S. Haddadin is with the Munich School of Robotics and Machine Intelligence and Chair of Robotics and System Intelligence, Technical University Munich, Germany, (e-mail: sami.haddadin@tum.de).}}

\maketitle
\thispagestyle{copyright}

\begin{abstract}
    Recently, many reactive trajectory planning approaches were suggested in the literature because of their inherent immediate adaption in the ever more demanding cluttered and unpredictable environments of robotic systems.
    However, typically those approaches are only locally reactive without considering global path planning and no guarantees for simultaneous collision avoidance and goal convergence can be given.
    In this paper, we study a recently developed \gls{cf}-based motion planner \cite{BeckerLilMulHad2021} that combines local reactive control with global trajectory generation by adapting an artificial magnetic field such that multiple trajectories around obstacles can be evaluated.
    In particular, we provide a mathematically rigorous analysis of this planner for static environments in the horizontal plane to ensure safe motion of the controlled robot.
    Contrary to existing results, the derived collision avoidance analysis covers the entire \gls{cf} motion planning algorithm including attractive forces for goal convergence and is not limited to a specific choice of the rotation field, i.e., our guarantees are not limited to a specific potentially suboptimal trajectory.
    Our Lyapunov-type collision avoidance analysis is based on the definition of an (equivalent) two-dimensional auxiliary system, which enables us to provide tight, if and only if conditions for the case of a collision with point obstacles.
    Furthermore, we show how this analysis naturally extends to multiple obstacles and we specify sufficient conditions for goal convergence.
    Finally, we provide challenging simulation scenarios with multiple non-convex point cloud obstacles and demonstrate collision avoidance and goal convergence.
\end{abstract}
\glsresetall

%
%
\section{Introduction}
\IEEEPARstart{N}{ew} technologies have enabled classical industrial robotics to be increasingly complemented and extended by more sensitive, lighter and less constrained robotic systems that no longer need to be operated behind safety fences.
Therefore, there has been a significant increase in research and application of human-robot collaborations in recent years \cite{AjoudaniZanIvaAlb2018}.
As a result, the use of robotic applications is naturally evolving away from structured and clearly delineated areas to unpredictable, cluttered and complex environments. This development poses new challenges, which place great demands on collision avoidance with obstacles and motion planning in particular \cite{KapplerDanielMeiIssMai2018}.
Traditional sense-plan-act approaches are reaching their limits, while reactive algorithms offer great potential due to their fast computing time and inherent immediate adaptation to unforeseen events \cite{KapplerDanielMeiIssMai2018,KragicGusKarJen2018}.\\
Research for mobile robotics and autonomous vehicles already achieved impressive results in those environments using reactive approaches \cite{KuwataTeoFioKar2009,LeonardHowTelBer2008,BachaBauFarFle2008}.
Nevertheless, motion planning remains to be an active topic in all application areas especially in terms of safety, where rigorous guarantees for collision avoidance are needed to ensure that a robot is able to perform its tasks safely even in the event of unforeseen situations \cite{HoyMatSav2015,WangSavGar2018}.
Such a rigorous analysis in in terms of goal convergence and collision avoidance is often neglected, in particular for reactive approaches where classical verification methods are not applicable \cite{HoyMatSav2015,SavkinWan2013}.\\
\emph{Related work:}
\Gls{ics} can be used to guarantee safe collision free motion planning by considering obstacle and robot dynamics to avoid states where a collision is unavoidable \cite{FraichardAsa2004}.
\gls{ics} result in a high computational complexity, which can be relaxed by planning partial \gls{ics}-free trajectories over a finite time \cite{PettiFra2005}.
Similarly, reachability analysis can be used for online verification of safe motions \cite{AlthoffDol2014,HerbertCheHanBan2017}.
Collision avoidance can also be verified using barrier certificates \cite{PrajnaJad2004}, or enforced using control barrier functions, which enable reactive control strategies with low computational complexity \cite{AmesCooEgeNot2019, WielandAll2007,RauscherKimHir2016}.
Other reactive algorithms that ensure collision avoidance include, e.g., the biologically inspired approach in \cite{TeimooriSav2010,SavkinWan2013} where an avoidance angle of the robots' velocity vector to the obstacles is used for collision avoidance. The algorithm was exploited and extended in several publications \cite{WangSavGar2018,WiigPetKro2017,WiigPetKro2018}.\\
\Glspl{apf} are one of the most popular reactive collision avoidance approaches which suffer from local minima, i.e., goal convergence cannot be guaranteed \cite{Khatib1986}.
Many variants of \glspl{apf} or related approaches were proposed to overcome the problem with local minima and to enable goal convergence in a wider range of applications, e.g., navigation vector fields \cite{RimonKod1988}, harmonic fields \cite{Masoud2010} or gyroscopic forces \cite{ChangMar2003}.
Collision avoidance guarantees for the gyroscopic force algorithm are available for planar robots \cite{ChangMar2003}, spherical or cylindrical obstacles in \gls{3d} space \cite{GarimellaSheKob2016}. In \cite{SabattiniSecFan2013} an additional breaking force was added to the definition of the gyroscopic force, ensuring collision avoidance at the expense of losing goal convergence guarantees.\\
Inspired by the interactions in magnetic fields, the reactive \gls{cf} algorithm first developed in \cite{SinghSteWen1996} has been increasingly studied in the last years \cite{HaddadinBelAlb2011,HaddadinParBelAlb2013,AtakaLamAlt2018,AtakaShiLamAlt2018,AtakaLamAlt2022,LahaFigVraSwi2021,LahaVorFigQu2021}.
Notably, this algorithm does not change the magnitude of the velocity, hence resulting in smooth trajectories without getting stuck in local minima.
Nevertheless, these approaches are typically only locally reactive and use a fixed rotation field to avoid obstacles, resulting in globally suboptimal paths.
In order to overcome this limitation, we developed the \gls{cfp} planner in \cite{BeckerLilMulHad2021}, which combines local reactive control with global motion planning and provides significantly improved trajectories compared to other reactive planners as it is able to avoid obstacles in multiple directions.\\
Among these magnetic field inspired collision avoidance algorithms, several analyses were conducted which are, to a certain extent, applicable to the \gls{cfp} planner \cite{BeckerLilMulHad2021}.
In the seminal work \cite{SinghSteWen1996}, goal convergence is shown in combination with an attractive force, assuming no collision occurs.
In addition, collision avoidance is studied under simplified assumptions, where the artificial magnetic field is kept constant (instead of changing with the movement of the robot) and without any attractive forces.
In \cite{AtakaLamAlt2018} collision avoidance in environments with a single convex obstacle is shown. The approach was recently enhanced in \cite{AtakaLamAlt2022} where an additional repulsive force was added that does not disturb the robot's velocity magnitude so that the collision avoidance guarantees could be extended to single nonconvex obstacles without conflicting with the goal convergence properties.
However, the existing approaches and analyses \cite{SinghSteWen1996,AtakaLamAlt2018,AtakaLamAlt2022} have two major limitations.
First, they are locally reactive by design and do not consider global trajectory planning. Thus, the analyses only provide collision avoidance guarantees for one possibly suboptimal trajectory.
Additionally, all previous collision avoidance analyses for \gls{cf} approaches were conducted for isolated \gls{cf} forces only, i.e., no collision avoidance guarantees could be given when the \gls{cf} force is combined with an additional attractive force which is necessary to achieve goal convergence.\\
\emph{Contribution:}
We address these issues by providing a rigorous mathematical analysis of the complete motion planning approach from \cite{BeckerLilMulHad2021} with \glspl{cf} for collision avoidance and an attractive potential force for goal convergence in planar environments.
In contrast to previous magnetic field inspired motion planners, the considered \gls{cfp} planner is able to generate multiple trajectories to avoid obstacles in different directions.
However, the existing analyses are not applicable to this setting.
In order to study this problem, we define auxiliary system dynamics in \cref{sec:collision_avoidance}, which can be used for a Lyapunov-type analysis of the \gls{cf} forces in the provided planar setting for all avoidance directions.
In this context, we show that a collision with a static point obstacle is only possible for initial conditions on a set of measure zero, which we characterise precisely.
Then, we show how the results can be naturally extended when the entire planner is used, i.e., \gls{cf} forces in combination with an additional goal force (\cref{sec:combined_force_analysis}) and that the previous guarantees remain valid under the influence of \gls{cf} forces from multiple obstacles (\cref{sec:multi_obstacles}). This is done by intermediately showing robustness with respect to small additional disturbances in \cref{sec:disturbance}.
We conclude the collision avoidance analysis with a qualitative argument for collision avoidance of point cloud obstacles (\cref{sec:point_clouds}).
Subsequently, we provide sufficient conditions for goal convergence using a potential field type argument (\cref{sec:goal_convergence}).
Finally, we demonstrate collision avoidance and goal convergence in a \gls{2d} simulation of a critical scenario with multiple nonconvex obstacles to highlight the theoretical results of this paper and in a \gls{3d} setting with dynamic obstacles and noisy point cloud data to show the practical capabilities of the \gls{cfp} planner (\cref{sec:simulations}).\\
\emph{Notation:}
In this paper, we use bold symbols to represent vectors, e.g., $\av \in \mathbb{R}^n$. Let $\av \times \bv$ be the cross product and $\av \cdot \bv$ be the dot product of two vectors $\av, \bv \in \mathbb{R}^n$. We denote the time derivative of a vector $\av \in \mathbb{R}^n$ by $\dot{\av}=\diff{\av}{t}$ and we define $\norm{\av}$ to be the Euclidean norm of this vector. We use $\land$ and $\lor$ to denote the logical conjunction and logical disjunction, respectively.
%
%
\section{Circular Field Motion Planner}\label{sec:cf_planner}
In the following, we describe the \gls{cf} motion planner from \cite{BeckerLilMulHad2021}.
We consider point mass dynamics and use a steering force $\Fm_{\mathrm{s}}$ for the calculation of the control signal, i.e., $\ddot{\xv} = \frac{\Fm_{\mathrm{s}}}{m}$, where $\ddot{\xv}\in\mathbb{R}^3$ is the robot acceleration and $m$ the robot mass.
Without loss of generality, we consider a unit mass $m=1$.
The steering force consists of an attractive potential field force $\Fm_\mathrm{VLC}$ for goal convergence and \gls{cf}-based obstacle avoidance forces $\Fm_{\mathrm{CF}}$ and is defined as
\begin{equation} \label{eq:total_force}
  \Fm_\mathrm{s} = \Fm_{\mathrm{CF}} + k_\mathrm{VLC} \Fm_\mathrm{VLC},
\end{equation}
where $k_\mathrm{VLC}\geq0$ is an additional scaling factor, which is explained in detail in \cref{sec:combined_force_analysis}.
Throughout this paper, we consider $j=1,\dots,n_o$ obstacles which are each characterized by a cloud of points $^j\pv_i\in\mathbb{R}^3, i=1,\dots, m_j$.
In this form, the obstacle data can be obtained from common motion tracking devices like laser scanners or cameras.
Moreover, we only consider static obstacles $\norm{^j\dot{\pv}_i} = 0$ and assume to have perfect knowledge of the environment, i.e., the position of each point $^j\pv_i$ of each obstacle is known exactly.
Extending the collision avoidance and goal convergence guarantees given in this paper to settings with dynamic obstacles and imperfectly known obstacle locations is an interesting subject for future work.
\subsection{Circular Field Force}
The \gls{cf} algorithm is inspired by the forces on moving charges in electromagnetic fields and in our formulation, each point $i$ on an obstacle $j$ generates its own artificial electromagnetic field.
Towards this end, we define an artificial current on the obstacle points as
\begin{equation}\label{eq:cf_current}
  ^j\cv_i = \frac{{}^j\dv_i}{\norm{{}^j\dv_i}} \times \bv_j,
\end{equation}
where $^j\dv_i= \xv - {}^j\pv_i$ is the distance vector between the robot's position $\xv$ and the position of the obstacle point $^j\pv_i$ and $\bv_j \in \mathbb{R}^3$ with $\norm{\bv_j}=1$ is the magnetic field vector of the obstacle that defines the rotation of the artificial magnetic field and thus the direction in which this obstacle is evaded.
Note that the magnetic field vector $\bv_j$ is set equal for all points $^j\pv_i$ on the same obstacle $i$ to prevent oscillations \cite{HaddadinBelAlb2011}.
Then, the artificial magnetic field from an obstacle point is defined as
\begin{equation} \label{eq:cf_b_field}
  ^j\Bm_i = \frac{\kcf}{\norm{^j\dv_i}} \, ^j\cv_i \times \frac{^j\dot{\dv}_i}{\norm{^j\dot{\dv}_i}},
\end{equation}
with the scaling factor $\kcf > 0$.
When the robot moves in such a magnetic field, the \gls{cf} force (a modified version of the Lorentz force) is generated, which prevents it from colliding with the obstacle point.
In order to save computational resources, we only apply the \gls{cf} force if the robot is in the vicinity $\dmax>0$ of this obstacle, i.e.,
\begin{equation} \label{eq:cf_force}
  ^j\Fm_{\mathrm{CF}, i} = \begin{cases}
    \frac{{}^j\dot{\dv}_i}{\norm{{}^j\dot{\dv}_i}} \times {}^j\Bm_i & \text{if } \norm{{}^j\dv_i} \leq \dmax \\
    0                                                               & \text{if } \norm{{}^j\dv_i} > \dmax
  \end{cases}.
\end{equation}
Finally, the resulting \gls{cf} force from $n_o$ obstacles each with $m_j$ points is a superposition of the individual forces of each obstacle point
\begin{equation} \label{eq:cf_force_sum}
  \Fm_{\mathrm{CF}} = \sum_{j=0}^{n_o} \sum_{i=0}^{m_j} {}^j\Fm_{\mathrm{CF}, i}.
\end{equation}
\subsection{Virtual Agents Framework}\label{sec:multi-agents}
The \gls{cfp} planner, first introduced in \cite{BeckerLilMulHad2021}, uses a virtual agents framework of predictive software agents to efficiently simulate different parameter settings in the currently known environment.
Towards this end, virtual agents generate multiple robot trajectories by using different magnetic field vectors $\bv_j$ in \cref{eq:cf_current} for computing the obstacle avoidance force $\Fm_{\mathrm{CF}}$ in \cref{eq:cf_force_sum} for an obstacle, e.g., one agent evades an obstacle on the left side and another agent on the right side.
New agent simulations are started after a defined cycle time and whenever a simulated agent comes close to a new obstacle the respective agent trajectory can be split by creating additional virtual agents with different magnetic field vectors.
By simulating multiple trajectories we can choose the best parameter set according to some specified cost (e.g., shortest path).
The prediction of trajectories is done in parallel to the calculation of the robot control command to ensure reactive behavior.
This parallelization provides a major advantage because it allows to calculate virtual agent trajectories and the next control signal for the real robot independently and asynchronously in separate computation threads.
As shown in the later theoretical exposition, obstacle avoidance does not require a specific choice for the magnetic field $\bv_j$ and is ensured due to the reactivity, even if the virtual agent framework fails to provide timely feedback on better available parameters in case of unexpected obstacles.
For details regarding the virtual agents framework, see \cite{BeckerLilMulHad2021}. In the context of the following theoretical analysis of the planner's obstacle avoidance capabilities, it is important to note that the magnetic field vector $\bv$ can be adjusted to modify the direction of the \gls{cf} force to allow for different trajectories around an obstacle.
Therefore, we investigate collision avoidance for general magnetic field vectors $\bv$ while existing analyses \cite{SinghSteWen1996,AtakaLamAlt2018,AtakaLamAlt2022} only analyze the case where $\bv$ is chosen such that the \gls{cf} force guides the robot around obstacles in the direction of its initial velocity (cf. \cref{lma:point_Sless0}).
Note that this implies that the following derived guarantees remain valid independently of the parameters provided by the virtual agents framework, which is an additional feature to improve the performance.
%
%
%
%
\section{Collision Avoidance Proof}  \label{sec:collision_avoidance}
In this section, we introduce the preliminaries (\cref{sec:preliminaries}), discuss the standing assumptions for our analysis (\cref{sec:assumptions}), and define an auxiliary system (\cref{sec:auxiliary_dynamics}), which we use for the collision avoidance analysis in \cref{sec:point_obstacle}.
\subsection{Preliminaries} \label{sec:preliminaries}
%
The following lemma shows that the \gls{cf} force does not affect the magnitude of the robot velocity.
\begin{lemma}\label{lma:constant_vel}
  The magnitude of the robot velocity $\norm{\dot{\xv}}$ is invariant under the dynamics $\ddot{\xv} = \Fm_{\mathrm{CF}}$ with $\Fm_{\mathrm{CF}}$ from \cref{eq:cf_force}.
\end{lemma}
\begin{proof}
  We adapt the idea of \cite{AtakaLamAlt2018} for the following proof.
  Given that the positions of all obstacle points are constant, i.e., $\diff{}{t} {}^j\pv_i = 0$, the relative velocity between each obstacle and the robot is equivalent to the velocity of the robot
  \begin{equation*}
    \diff{}{t} {}^j\dv_i= \diff{}{t} \xv - \diff{}{t} {}^j\pv_i = \diff{}{t}\xv.
  \end{equation*}
  Then, the magnitude of the robot velocity $\norm{\dot{\xv}}$ stays constant since the force  $\Fm_{\mathrm{CF}}$ always acts perpendicular, i.e.,
  \begin{align}
    \diff{}{t}\frac{\norm{\dot{\xv}}^2}{2}= \dot{\xv}\cdot & \ddot{\xv}= \dot{\xv}\cdot  \Fm_{\mathrm{CF}}                                                                                      \nonumber             \\
                                                           & = \dot{\xv} \cdot \sum_{j=0}^{n_o} \sum_{i=0}^{m_j} \left( \frac{{}^j\dot{\dv}_i}{\norm{{}^j\dot{\dv}_i}} \times {}^j\Bm_i \right) \label{eq:dotx_Fcf_0} \\
                                                           & = \dot{\xv} \cdot \left( \frac{\dot{\xv}}{\norm{\dot{\xv}}} \times \sum_{j=0}^{n_o} \sum_{i=0}^{m_j} {}^j\Bm_i \right) = 0. \nonumber \qedhere
  \end{align}
\end{proof}
%
\subsection{Assumptions for the Analysis} \label{sec:assumptions}
For the remainder of the paper, we consider a point-like robot with $\dot{\xv}(0) \neq 0$, which according to \cref{lma:constant_vel} ensures $\dot{\xv}(t) \neq 0$ for all $t \geq 0$ for the dynamics $\ddot{\xv}=\Fm_{\mathrm{CF}}$.
For the sake of clarity, in the current \cref{sec:collision_avoidance} we first provide a proof without a goal force in a planar scenario\footnote{Note that if the magnetic field vector $\bv$ is orthogonal to the plane that is spanned by $\xv$ and $\dot{\xv}$, the resulting \gls{cf} force only works in this $\xv, \dot{\xv}$ plane, i.e., it does not affect the velocity parts of the robot in any other direction.
  Hence, our following analysis for the \gls{2d} setting also applies to this special case in any \gls{3d} setting, which can always be enforced by choosing $\bv$ suitably.\label{fn:3d_analysis}}, i.e., $x_3 = 0$.
Additionally, we only consider a single point-like obstacle in the origin, i.e., $\dv = \xv$ as shown in \cref{fig:scenario}.
\begin{figure}%
  \centering
  \includegraphics[width=0.8\columnwidth]{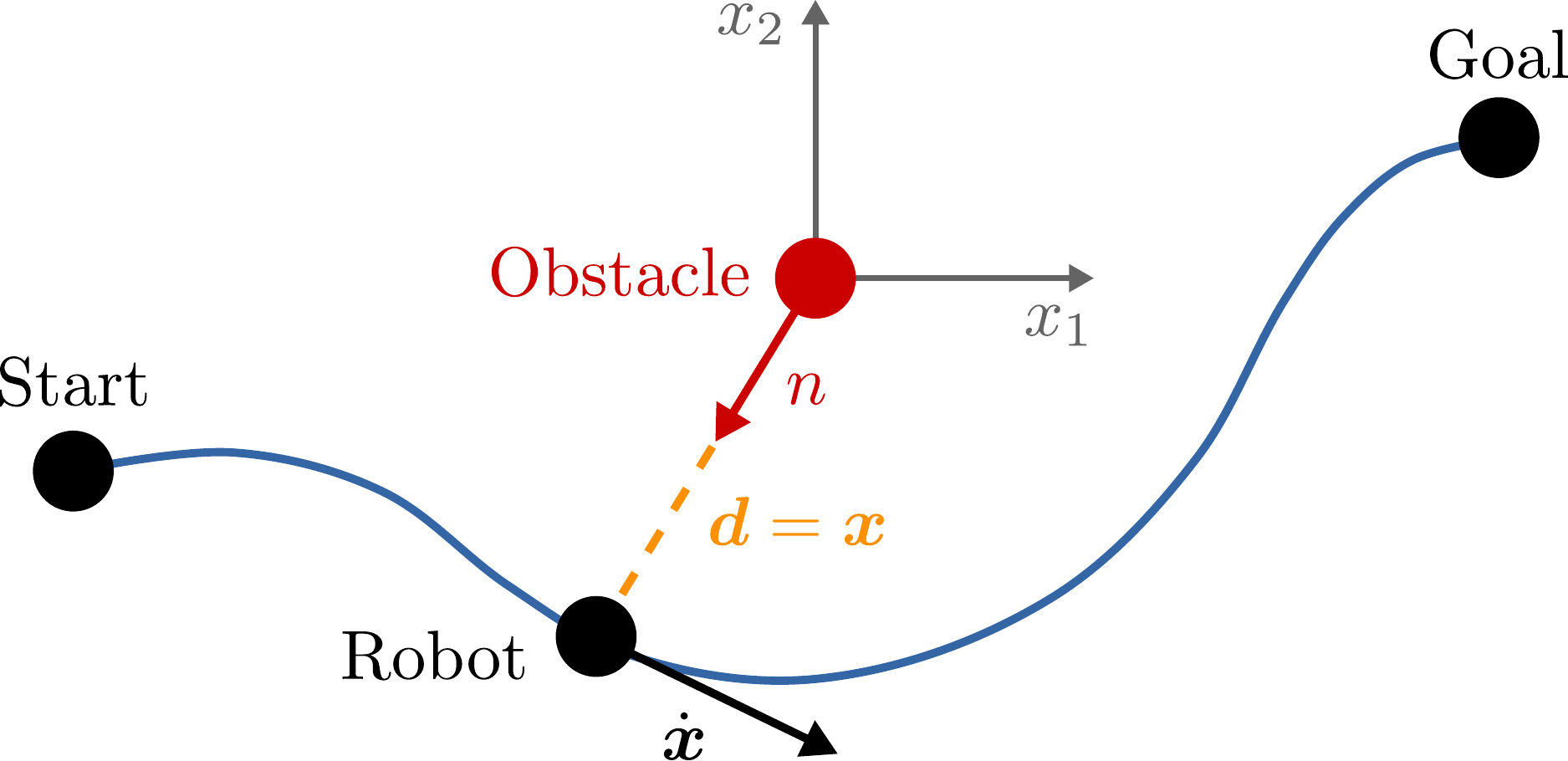}%
  \caption{Schematic view of obstacle avoidance scenario with an exemplary path depicted in blue.}%
  \label{fig:scenario}%
\end{figure}
In such a \gls{2d} scenario, there exist only two possible options for a collision free trajectory, to pass the obstacle on the left or on the right side. Hence, the only choices for the magnetic field vector are $\bv=\left(0 \, 0 \, 1\right)^T$ or $\bv=\left(0 \, 0 \, -1\right)^T$.
Note that in the \gls{2d} scenario, $\bv$ is always defined orthogonal to $\xv$ and $\dot{\xv}$, i.e., $\bv \cdot \xv = \bv \cdot \dot{\xv} = 0$ and therefore the following holds
\begin{align}
  \ddot{\xv} & \overset{\eqref{eq:cf_current}-\eqref{eq:cf_force}}{=} \frac{\kcf}{\norm{\dot{\xv}}^2\norm{\xv}^2} \dot{\xv} \times \left( \left( \xv \times \bv \right) \times \dot{\xv} \right) \nonumber \\
             & = \frac{\kcf}{\norm{\dot{\xv}}^2\norm{\xv}^2} \dot{\xv} \times \left( \bv \left(\dot{\xv} \cdot \xv \right) - \xv\left(\bv \cdot \dot{\xv} \right) \right) \nonumber                        \\
             & = \frac{\kcf}{\norm{\dot{\xv}}^2\norm{\xv}^2} \left(\dot{\xv} \times  \bv \right)\left(\dot{\xv} \cdot \xv \right) \label{eq:ddx}                                                           \\
             & = \pm \frac{\kcf}{\norm{\xv}^2 \norm{\dot{\xv}}^2} \begin{bmatrix}
    x_2 \dot{x}_2^2  + x_1 \dot{x}_1 \dot{x}_2   \\
    - x_1 \dot{x}_1^2  - x_2 \dot{x}_1 \dot{x}_2 \\
    0
  \end{bmatrix}, \nonumber
\end{align}
where the second equality follows from the triple product expansion and the third one from the fact that $\bv \cdot \dot{\xv} = 0$.
The extensions to environments with multiple obstacles, point cloud obstacles and the combination with an attractive force are discussed in \cref{sec:multi_obstacles}, \cref{sec:point_clouds} and \cref{sec:combined_force_analysis}, respectively.
\begin{remark}\label{rmk:theory_vs_appl}
  Note that the ensuing theoretical analysis is restricted to point mass robots in the presence of known stationary obstacles within a \gls{2d} setting (or a \gls{3d} setting, provided the magnetic field vector is orthogonal to the plane spanned by $\xv$ and $\dot{\xv}$, compare \cref{fn:3d_analysis}).
  However, it is worth emphasizing that the planner is purposely designed to function seamlessly within \gls{3d} environments featuring dynamic obstacles, although its efficacy in such situations has only been established empirically thus far (cf. \cref{sec:simulations} and \cite{BeckerLilMulHad2021,BeckerCasHatLilHadMul2023}).
  Note that obstacles outside a range $\dmax$ are not considered by the planner (cf. \cref{eq:cf_force}).
  Hence, only partial knowledge regarding the environment is required for implementation, i.e., the theoretical guarantees remain valid in partially known environments as long as obstacles in the range $\dmax$ around the robot are known.
\end{remark}
\subsection{Auxiliary Dynamics} \label{sec:auxiliary_dynamics}
For the following collision avoidance analysis of a robot controlled by the \gls{cf} force defined in \cref{eq:cf_force}, the definition of an auxiliary system will be crucial. To this end, we define the auxiliary system states
\begin{align}
  R & = \xv \cdot \dot{\xv}, \label{eq:def_r}                           \\
  S & = \left( \xv \times \dot{\xv} \right) \cdot \bv, \label{eq:def_s}
\end{align}
which also imply $R^2+S^2 = \norm{\xv}^2\norm{\dot{\xv}}^2$.
Intuitively, the sign of $R$ describes the moving direction of the robot with respect to the obstacle, i.e., $R>0$ represents the case when the robot is moving away from the obstacle. Similarly, the sign of $S$ describes if the robot is already moving in the intended avoidance direction described by $\bv$, e.g., $S<0$ implies that the robot's velocity points in the same direction as the circular motion the \gls{cf} force tries to achieve.
Accordingly, the critical case in the collision avoidance analysis corresponds to $R<0$ and $S>0$ at the same time because the robot moves towards the obstacle and the avoidance force must first change the direction of the robot's velocity in order to guide the robot around the obstacle in the intended direction (as shown in \cref{fig:RS_cases}).
Note that previous analyses in the literature only consider the case $S<0$ \cite{AtakaLamAlt2018,AtakaLamAlt2022,SinghSteWen1996}.\\
\begin{figure}%
  \centering
  \includegraphics[width=0.85\columnwidth]{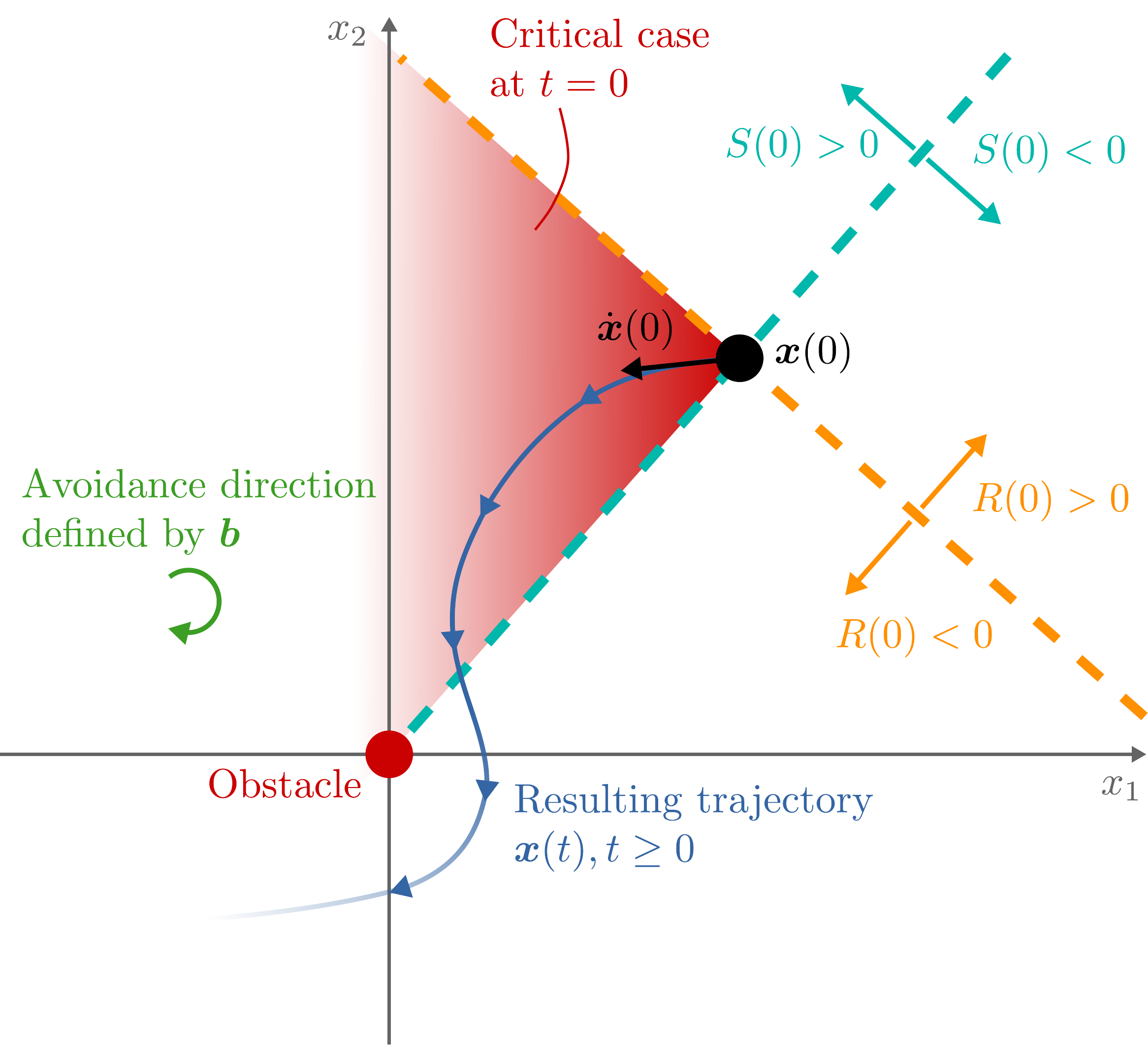}%
  \caption[Visualization of critical trajectory]{Visualization of the critical case ($R(0)<0, S(0)>0$) at $t=0$ as a dependency of the initial robot velocity $\dot{\xv}(0)$ (black), given an initial robot position $\xv(0)=\dv(0)$ and given the magnetic field vector $\bv = \vect{0 & 0 & 1}^T$.
    The resulting trajectory $\xv(t), t\geq 0$ is depicted in blue. Additionally, the sign of $R(0)$ (orange) and $S(0)$ (teal) depending on $\dot{\xv}(0)$ is shown.}
  \label{fig:RS_cases}%
\end{figure}
The derivatives of the auxiliary states $S$, $R$ are given by
\begin{align}
  \dot{R} & = \xv \cdot \ddot{\xv} + \dot{\xv} \cdot \dot{\xv}  \nonumber                                                                                               \\
          & = \frac{\kcf}{\norm{\xv}^2\norm{\dot{\xv}}^2} \xv \cdot \left(\dot{\xv} \times \bv\right) \left(\dot{\xv} \cdot \xv \right) + \norm{\dot{\xv}}^2 \nonumber  \\
          & = \frac{\kcf}{\norm{\xv}^2\norm{\dot{\xv}}^2}  \left(\xv \times\dot{\xv}\right) \cdot \bv \left( \xv \cdot \dot{\xv} \right) + \norm{\dot{\xv}}^2 \nonumber \\
          & = \frac{\kcf}{\norm{\xv}^2\norm{\dot{\xv}}^2} R S + \norm{\dot{\xv}}^2  = \kcf \frac{R S}{R^2+S^2} + \norm{\dot{\xv}}^2, \label{eq:dot_r}
\end{align}
\begin{align}
  \dot{S} & = \left( \xv \times \ddot{\xv}\right) \cdot \bv \nonumber                                                                                                                                                     \\
          & = \frac{\kcf}{\norm{\xv}^2\norm{\dot{\xv}}^2} \left[\xv \times \left(\dot{\xv} \times \bv\right) \left(\dot{\xv} \cdot \xv \right)\right] \cdot \bv \nonumber                                                 \\
          & = \frac{\kcf}{\norm{\xv}^2\norm{\dot{\xv}}^2} \left[ \left( \left( \xv \cdot \bv \right)\dot{\xv} - \left(\xv \cdot \dot{\xv}\right) \bv \right) \left(\dot{\xv} \cdot \xv \right)\right] \cdot \bv \nonumber \\
          & = \frac{\kcf}{\norm{\xv}^2\norm{\dot{\xv}}^2} \left[ \left( - \left(\xv \cdot \dot{\xv}\right) \bv \right) \left(\dot{\xv} \cdot \xv \right)\right] \cdot \bv \nonumber                                       \\
          & = -\frac{\kcf}{\norm{\xv}^2\norm{\dot{\xv}}^2} \left(\xv \cdot \dot{\xv}\right)^2 \left(\bv \cdot \bv\right)  = -\kcf \frac{R^2}{R^2+S^2}, \label{eq:dot_s}
\end{align}
where we use \cref{eq:ddx} and the properties of the triple product and the vector triple product to transform their derivatives.
We use the concept of barrier functions (compare, e.g., \cite{AmesCooEgeNot2019}) to analyze collision avoidance. To this end, we define the barrier function $V_B$ as
\begin{align}
  V_B           & = \frac{1}{\norm{\dv}^2} = \frac{1}{\norm{\xv}^2} = \frac{\norm{\dot{\xv}}^2}{R^2 + S^2}, \label{eq:barrier_ab}                                                                \\
  \diff{}{t}V_B & = -2\norm{\xv}^{-3} \diff{}{t} \norm{\xv} = -2 \frac{\xv \cdot \dot{\xv}}{\norm{\xv}^{4}} = -2 \frac{R \norm{\dot{\xv}}^4}{\left(R^2 + S^2\right)^2} \label{eq:dot_barrier_rs}
\end{align}
and note that $V_B(t) < \infty$ for all $t \geq 0$ implies that there is no collision.
One of the main points we would like to emphasize, which is crucial for the following analysis, is that for a given (constant) $\norm{\dot{\xv}}\neq0$ (cf. \cref{lma:constant_vel}), $R$ and $S$ describe a two-dimensional nonlinear autonomous system. Furthermore, with \cref{eq:barrier_ab}, we will use this representation to analyze collision avoidance.
In particular, a collision occurs if and only if $R=S=0$.
Correspondingly, despite the discontinuity of the dynamics \eqref{eq:dot_r} and \eqref{eq:dot_s} at $S=R=0$, this autonomous system is well-defined if no collision occurs.
Note that given a uniform bound $V_B(t)\leq c_1<\infty, c_1>0$ (and $\norm{\dot{\xv}} \leq c_2 <\infty, c_2>0$), we get a uniform bound on the steering force $\Fm_\mathrm{s}$ and the acceleration $\ddot{\xv}$ (cf. \cref{eq:vlc_force,eq:cf_b_field,eq:cf_force,eq:total_force}).
%
\subsection{Collision Avoidance with a Pointlike Obstacle} \label{sec:point_obstacle}
In the following, we show the collision avoidance properties of the \glspl{cf}, i.e., that $(R(t),S(t))\neq (0,0)$ holds for all $t \geq 0$.
An exemplary vector field of the $R$-$S$ dynamics can be seen in \cref{fig:rs_dynamics}.
It shows the phase plot of $R$ and $S$ as stated in \cref{eq:dot_r,eq:dot_s} for some fixed $\norm{\dot{\xv}}, \kcf > 0$.
In the following, we show collision avoidance using a case distinction depending on the quadrant of the initial condition $R(0)$, $S(0)$.
This case distinction is also illustrated in \cref{fig:rs_dynamics} where we colored the quadrants according to the corresponding lemmas.
Additionally, this figure shows that there exist initial conditions which inevitably lead to a collision.
Note that these conditions are located on a ray in the direction of the origin (shown in red in \cref{fig:rs_dynamics}), which we therefore term \emph{collision ray} in the remainder of this paper.
Mathematically, the \emph{collision ray} is characterized by $S(t) + c R(t) = 0$ with $c=\frac{\kcf}{\norm{\dot{\xv}}^2}$.
A rigorous proof and analysis of these conditions are given in \cref{lma:point_critical}.
Please note that in the following analysis, we assume initial conditions that are collision free, i.e., $V_B (0) < \infty$, which implies $(R(0), S(0)) \neq (0, 0)$ (the origin in \cref{fig:rs_dynamics}).
In particular, \cref{lma:point_Rgeq0,lma:point_Sless0,lma:point_critical} cover all the different possible initial conditions, which together ensures collision avoidance for (almost) all initial conditions, compare \cref{thm:ca_point}.
\begin{figure}%
  \centering
  \includegraphics[width=0.7\columnwidth]{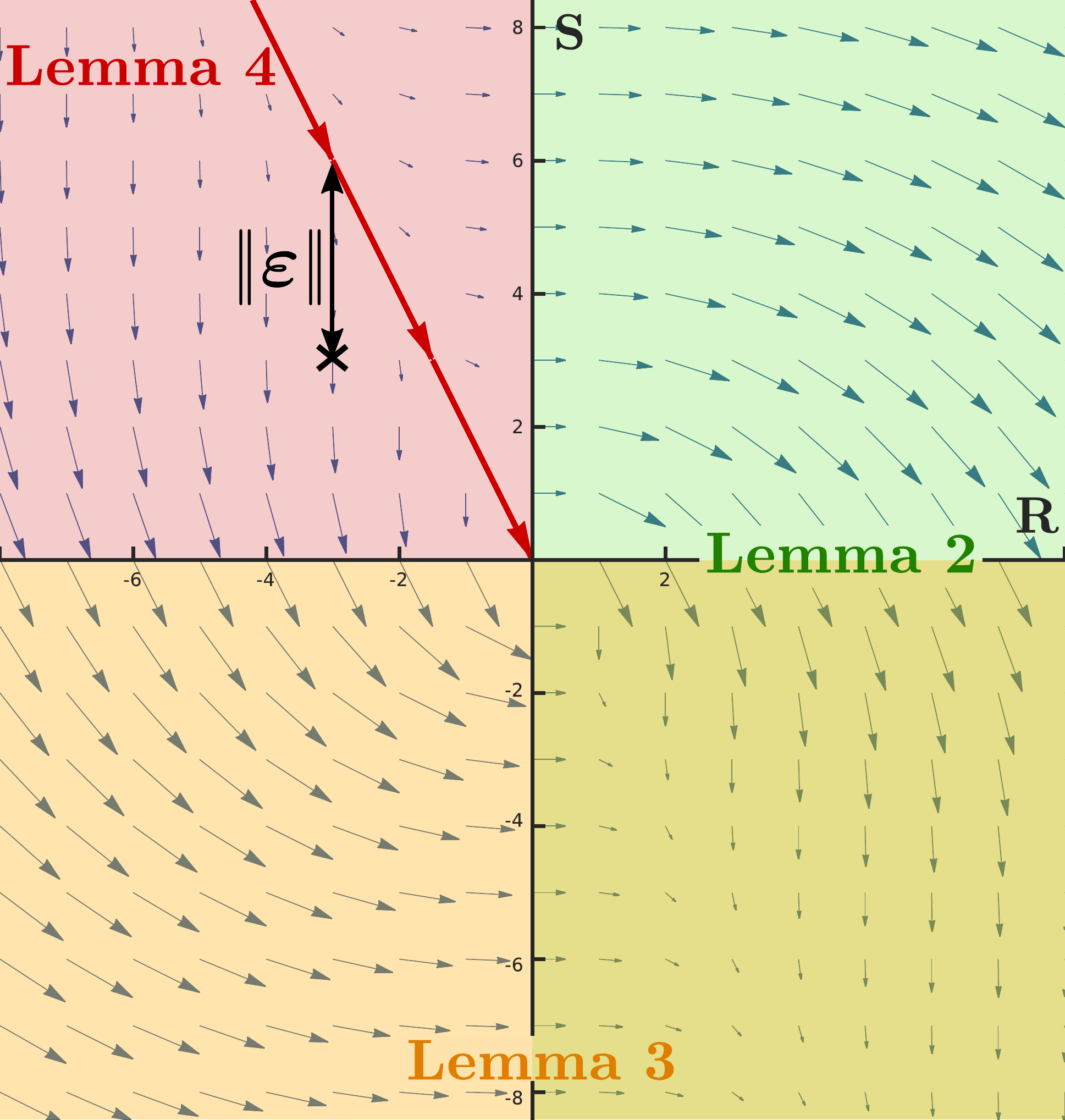}%
  \caption{Vector field plot that shows the $RS$ dynamics according to \cref{eq:dot_r,eq:dot_s}, where the arrows indicate the changing trends of $R$ and $S$.
    Initial conditions with $R(0)\geq 0$ are studied in \cref{lma:point_Rgeq0} (green). \Cref{lma:point_Sless0} considers initial conditions with $S(0)<0$ (orange). Note that both \cref{lma:point_Rgeq0,lma:point_Sless0} cover initial conditions in the lower right quadrant. \Cref{lma:point_critical} investigates initial conditions where $R(0)<0$ and $S(0)>0$ (red).
    Additionally, the \emph{collision ray} of all initial conditions that lead to an inevitable collision is depicted in red. The distance in $S$-direction from the collision ray, $\norm{\epsilon}$ (compare \cref{eq:def_epsilon}), is shown with a black arrow for exemplary initial conditions at $R(0)=-3$, $S(0)=3$.}%
  \label{fig:rs_dynamics}%
\end{figure}
We start the analysis with the simplest case, where the robot is already moving away from the obstacle.
%
\begin{lemma}\label{lma:point_Rgeq0}
  For any $R(0) \geq 0$, the dynamics in \cref{eq:dot_r,eq:dot_s} yield $R(t) \geq 0$ and $V_B(t) \leq V_B(0)$ for all $t \geq 0$.
\end{lemma}
%
\begin{proof}
  To show that $R(t)\geq 0$ holds recursively it suffices to show that $\dot{R}(t)\geq 0$ if $R(t)=0$, which holds with $\dot{R} = \norm{\dot{\xv}}^2 \geq 0$.
  Furthermore, using $R(t)\geq 0$ in \cref{eq:dot_barrier_rs}, we have $\dot{V}_B(t) \leq 0$ for all $t \geq 0$ and hence $V_B(t) \leq V_B(0)$ for all $t \geq 0$.
\end{proof}
In the next lemma, we consider initial conditions where the robot is already following the intended direction around the obstacle (by definition of the magnetic field vector).
%
%
\begin{lemma}\label{lma:point_Sless0}
  For any $S(0) < 0$, the dynamics in \cref{eq:dot_r,eq:dot_s} yield $S(t)^2 \geq S(0)^2$ and $V_B(t) \leq \frac{\norm{\dot{\xv}}^2}{S(0)^2}<\infty$ for all $t \geq 0$.
\end{lemma}
%
\begin{proof}
  \Cref{eq:dot_s} ensures $\dot{S}(t) \leq 0$.
  Hence, given $S(0) < 0$, we have $S(t)^2 \geq S(0)^2 > 0$ for all $t \geq 0$.
  Therefore, using \cref{eq:barrier_ab} yields $V_B(t) \leq \frac{\norm{\dot{\xv}}^2}{S(0)^2}$ for all $t \geq 0$.
\end{proof}
%
The result in~\cref{lma:point_Sless0} deteriorates for $|S(0)|$ arbitrary small. However, uniform bounds for $V_B(t)$ in the form $V_B(t)\leq k_2V_B(0)$ with $k_2>0$ can also be derived, compare \cref{sec:proof_remark}.\\
%
%
For the analyses of the critical case ($R<0, S>0$) we define
\begin{equation}
  \epsilon(t) = S(t)+c R(t) \label{eq:def_epsilon}
\end{equation}
with $c=\frac{\kcf}{\norm{\dot{\xv}}^2}$, which corresponds to the distance to the collision ray in $S$-direction (cf. \cref{fig:rs_dynamics}).
Note that $\norm{\dot{\xv}}$ constant (cf. \cref{lma:constant_vel}) implies that $c$ is constant and thus $\dot{\epsilon} = \dot{S} +c \dot{R}$.
Substituting \cref{eq:dot_r,eq:dot_s}, the time derivative of \cref{eq:def_epsilon} is given by
\begin{align}
  \dot{\epsilon} & = \dot{S} +c \dot{R} \nonumber                                                                     \\
                 & = -\kcf \frac{R^2}{R^2+S^2} + c\kcf \frac{R S}{R^2+S^2} + c\norm{\dot{\xv}}^2\nonumber             \\
                 & = \kcf\frac{-R^2 + cRS}{R^2+S^2} + \kcf \nonumber                                                  \\
                 & = \kcf\frac{S \left(S + cR \right)}{R^2+S^2} = \kcf \frac{S}{R^2+S^2} \epsilon. \label{eq:dot_eps}
\end{align}
Upon inspection of \cref{eq:dot_eps}, it becomes evident that for $S>0$, the collision ray, i.e., $\epsilon=0$, is non-attractive and hence any initial condition $\epsilon \neq 0$ ensures collision avoidance.
This intuition is formally proved in the following lemma.
\begin{lemma}\label{lma:point_critical}
  For any $R(0)<0, S(0) > 0$ and $\epsilon(0)\neq 0$, there exists a time $\tau>0$ such that $R(\tau)=0$ or $S(\tau)=0$ and for all $t\in[0,\tau]$:  $|\epsilon(t)|\geq|\epsilon(0)|$.
  In addition, we have $\norm{\xv(t)}\geq \frac{ |\epsilon(0)|}{\max(c,1) \norm{\dot{\xv}}}>0$, $t\in[0,\tau]$, i.e., no collision occurs.
  Furthermore, in case $\epsilon(0)=0$, there exists a time $\tau>0$ such that $\epsilon(t)=0$ for all $t\in[0,\tau)$ and $\lim_{t\rightarrow\tau} \norm{\xv(t)}=0$, i.e., we have a collision.
\end{lemma}
%
\begin{proof}
  The following proof is split into three parts. First, we show that we leave the critical quadrant in finite time, i.e., for each $R(0)<0$ and $S(0)>0$, there exists a constant $\tau_{\max}>0$, such that
  \begin{equation}\label{eq:def_tau_max}
    \tau_{\max} \geq  \tau:= \inf_{t\geq 0, S(t)=0 \lor R(t)=0} t.
  \end{equation}
  Then we show $|\epsilon(t)|\geq |\epsilon(0)|$, $t\in[0,\tau]$. Finally, we prove a lower bound on $\norm{\xv}$ for $\epsilon(0)\neq 0$.\\
  %
  \textbf{Part I.}
  In the following, we distinguish three cases: $\epsilon(0)<0$, $\epsilon(0)>0$, $\epsilon(0)=0$.
  We show that for all three cases, there exists a time $\tau$ such that $R(\tau)=0$ or $S(\tau)=0$.\\
  We would like to point out that in the following, in \emph{case i)} and \emph{case ii)}, i.e., $\epsilon(0)<0$ and $\epsilon(0)>0$ we repeatedly use
  \begin{equation} \label{eq:eps_geq_eps0}
    |\epsilon(t)| \geq |\epsilon(0)| > 0,
  \end{equation}
  for all $t\in[0,\tau]$, which will be recursively established in Part~II.\\
  %
  \emph{Case i)}: $\epsilon(0) <0$:
  From \cref{eq:eps_geq_eps0} and $\epsilon(0) <0$, it follows that $\epsilon(t) <0$ for all $t\in [0,\tau]$,
  which implies that $\tau$ in \cref{eq:def_tau_max} is such that $S(\tau)=0$ and $R(t)<0$ for all $t\in [0,\tau]$.
  Furthermore, by \cref{eq:def_epsilon} and $S(t)\geq0$ it follows that $\frac{S^2}{R^2}<c^2$ for all $t\in [0,\tau]$.
  Correspondingly, using \cref{eq:dot_s}, we have
  \begin{equation} \label{eq:bound_dotS_below}
    \dot{S}=-\kcf \frac{1}{1+\frac{S^2}{R^2}}< - \frac{\kcf}{1+c^2}.
  \end{equation}
  For contradiction, suppose $\tau>\tau_{\max,1}$ with
  \begin{equation}
    \tau_{\max,1}: = \frac{S(0) \left(1+c^2\right)}{\kcf}. \label{eq:tmax_below}
  \end{equation}
  Then, integration of \cref{eq:bound_dotS_below} yields
  \begin{align*}
    S(\tau_{\max,1}) & \leq S(0) - \frac{\kcf}{1+c^2}\tau_{\max,1}                           \\
                     & \leq S(0) - \frac{\kcf}{1+c^2}\frac{S(0) \left(1+c^2\right)}{\kcf}=0,
  \end{align*}
  which contradicts $S(t)> 0$, $t\in[0,\tau)$. Hence, by contradiction we leave the upper left quadrant with $\tau\leq\tau_{\max,1}$.\\
  %
  \emph{Case ii)}: $\epsilon(0) >0$:
  Note that \cref{eq:eps_geq_eps0} and $\epsilon(0) >0$ imply $\epsilon(t) >0$ and hence $S(t)>0$ for all $t\in [0,\tau]$.\\
  \emph{Case ii) a)}: $c\geq1$:
  Applying \cref{eq:bound_ab-term} from \cref{lma:bound_ab-term} to the first term of \cref{eq:dot_r} yields
  \begin{equation} \label{eq:bound_dotR_above_cgeq1}
    \dot{R}(t) \geq  - \kcf \frac{c}{c^2+1} + \frac{\kcf}{c} = \frac{\kcf}{c+c^3} > 0.
  \end{equation}
  For contradiction, suppose $\tau>\tau_{\max,2a}$ with
  \begin{equation}
    \tau_{\max,2a}: = \frac{- R(0) \left( c + c^3\right)}{\kcf} \label{eq:tmax_above_1}.
  \end{equation}
  Then, integration of \cref{eq:bound_dotR_above_cgeq1} yields
  \begin{align*}
    R(\tau_{\max,2a}) & \geq R(0) + \frac{\kcf}{c+c^3}\tau_{\max,2a}                                  \\
                      & \geq R(0) + \frac{\kcf}{c+c^3} \frac{- R(0) \left( c + c^3\right)}{\kcf} = 0,
  \end{align*}
  which contradicts $R(t)< 0$, $t\in[0,\tau)$. Hence, by contradiction we leave the upper left quadrant with $\tau\leq\tau_{\max,2a}$.\\
  \emph{Case ii) b)}: $c<1$:
  Analogously, applying \cref{eq:bound_ab-term} from \cref{lma:bound_ab-term} to the first term of \cref{eq:dot_r} yields
  \begin{align}
    \dot{R}(t) & \geq   - \frac{\kcf}{2} + \frac{\kcf}{c} =\kcf \frac{2-c}{2c} \nonumber \\
               & \geq \frac{\kcf}{2} > 0. \label{eq:bound_dotR_above_cless1}
  \end{align}
  For contradiction, suppose $\tau>\tau_{\max,2b}$ with
  \begin{equation}
    \tau_{\max,2b}:  = -\frac{ 2 R(0)}{\kcf} \label{eq:tmax_above_2}.
  \end{equation}
  Then, integration of \cref{eq:bound_dotR_above_cless1} yields
  \begin{align*}
    R(\tau_{\max,2b}) & \geq R(0) + \frac{\kcf}{2}\tau_{\max,2b}             \\
                      & \geq R(0) + \frac{\kcf}{2}\frac{ -2 R(0)}{\kcf} = 0,
  \end{align*}
  i.e., by contradiction we leave the upper left quadrant at some time $\tau\leq \tau_{\max,2b}$.\\
  %
  \emph{Case iii)}: $\epsilon(0)=0$:
  From \cref{eq:dot_eps}, we know that the linear subspace $\epsilon=S+cR = 0$ is positively invariant.
  Additionally, from $R^2=\frac{S^2}{c^2}$ and \cref{eq:dot_s}, it follows that $S$ is linearly decreasing with
  \begin{align}
    \dot{S} & = -\kcf \frac{R^2}{R^2+S^2}        \nonumber      \\
            & = -\kcf \frac{1}{1+c^2} \label{eq:dotS_collision}
  \end{align}
  for all $t\in[0,\tau)$. Correspondingly, the integration of \cref{eq:dotS_collision} in the interval $[0,\tau)$ with
  \begin{equation}
    \tau  =  \frac{S(0) \left(1+c^2\right)}{\kcf} \label{eq:tmax_collision}
  \end{equation}
  yields
  \begin{align*}
    S(\tau) & = S(0) - \frac{\kcf}{1+c^2}\tau                                       \\
            & = S(0) - \frac{\kcf}{1+c^2} \frac{S(0) \left(1+c^2\right)}{\kcf} = 0,
  \end{align*}
  which also implies $R(\tau)=0$.\\
  %
  \textbf{Part II.}
  In order to show that $|\epsilon(t)|\geq |\epsilon(0)|$ holds, we use that $\kcf \frac{S}{R^2+S^2} \geq 0$ holds for all $t\in[0,\tau]$. Therefore, from \cref{eq:dot_eps}, it follows that $\diff{}{t}|\epsilon(t)|\geq 0$ and thus
  \begin{equation}
    |\epsilon(t)|\geq |\epsilon(0)| \label{eq:eps_geq_eps0_proof}
  \end{equation}
  for all $t\in [0,\tau]$.\\
  %
  \textbf{Part III.}
  In the following, we show a lower bound for the robot obstacle distance for $|\epsilon(0)|>0$.
  Additionally, we show that the initial condition $\epsilon(0) = 0$ inevitably leads to a collision with the obstacle.\\
  %
  \emph{Case i)}: $\epsilon(0) <0$:
  Given $\epsilon(0) <0$, \cref{eq:eps_geq_eps0_proof} implies $\epsilon(t) \leq \epsilon(0)$ for all $t\in[0,\tau]$. Then \cref{eq:def_epsilon} yields $R(t) \leq \frac{\epsilon(0)-S(t)}{c} \leq \frac{\epsilon(0)}{c}$ and thus $R(t)^2 \geq \frac{\epsilon(0)^2}{c^2}$. Combining this inequality with \cref{eq:barrier_ab} establishes the following bound for the barrier function $V_B(t)\leq\frac{c^2 \norm{\dot{\xv}}^2}{\epsilon(0)^2} = \frac{\kcf^2}{\norm{\dot{\xv}}^2 \epsilon(0)^2}$, where the last equality follows from recalling that $c=\frac{\kcf}{\norm{\dot{\xv}}^2}$.
  Therefore, using \cref{eq:barrier_ab}
  \begin{equation}\label{eq:bound_x_point_eps_less}
    \norm{\xv(t)} \geq \frac{\norm{\dot{\xv}} |\epsilon(0)|}{\kcf}=\frac{|\epsilon(0)|}{c \norm{\dot{\xv}}},\quad t\in[0,\tau].
  \end{equation}
  %
  \emph{Case ii)}: $\epsilon(0) >0$:
  Analogously, given $\epsilon(0) >0$, \cref{eq:eps_geq_eps0_proof} implies $\epsilon(t) \geq \epsilon(0)$ for all $t\in[0,\tau]$. Then \cref{eq:def_epsilon} yields $S(t) \geq \epsilon(t) \geq \epsilon(0)$ and thus $S(t)^2 \geq \epsilon(0)^2$. This can be used with \cref{eq:barrier_ab} to establish the following bound for the barrier function $V_B(t)\leq\frac{\norm{\dot{\xv}}^2}{\epsilon(0)^2}$ and therefore
  \begin{equation} \label{eq:bound_x_point_eps_greater}
    \norm{\xv(t)} \geq \frac{|\epsilon(0)|}{\norm{\dot{\xv}}},\quad t \in [0,\tau].
  \end{equation}
  %
  \emph{Case iii)}: $\epsilon(0) =0$:
  Given $\epsilon(0) = 0$, \cref{eq:dot_eps} implies $\epsilon(t) = \epsilon(0)=0$ for all $t\in[0,\tau]$. Then \cref{eq:def_epsilon} yields $cR(t) = S(t)$ and thus $R(\tau)= S(\tau) = 0$. Therefore, $V_B(\tau) = \infty$ and $\norm{\xv(\tau)}=0$, i.e., the robot collides with the obstacle.
\end{proof}
%
%
As stated before, \cref{lma:point_Rgeq0,lma:point_Sless0,lma:point_critical} combined cover all possible initial conditions and therefore guarantee collision avoidance except for a set of initial conditions of measure zero.
\begin{theorem}\label{thm:ca_point}
  For the dynamics $\ddot{\xv}=\Fm_{\mathrm{CF}}$ with $\Fm_{\mathrm{CF}}$ according to \cref{eq:cf_force} and with a point obstacle, no collision occurs for almost all initial conditions.
  In particular, a collision occurs if and only if the initial condition satisfies $R(0)<0$, $S(0)>0$ and $\epsilon(0)=0$.
\end{theorem}
%
\begin{proof}
  Follows by combining cases in \cref{lma:point_Rgeq0,lma:point_Sless0,lma:point_critical} (cf. \cref{fig:rs_dynamics}).
\end{proof}
The derived analysis can also be used constructively in the critical case (\cref{lma:point_critical}) if we are close to the collision ray.
In particular, in case $R<0,S>0$, we can redefine
\begin{equation} \label{eq:scaling_kcf}
  \tilde{k}_\mathrm{CF} = \kcf - \mathrm{sgn}(\epsilon) \left(\epsilon_\mathrm{min} - |\epsilon|\right) \frac{\norm{\dot{\xv}}^2}{|R|} \quad \text{if } |\epsilon| < \epsilon_\mathrm{min},
\end{equation}
where $0<\epsilon_\mathrm{min}$ is a scaling factor that defines the threshold for the distance to the collision ray and should be chosen small to avoid frequent changes of $\kcf$. By using $\tilde{k}_\mathrm{CF}$ instead of $\kcf$ we ensure $|\epsilon| \geq \epsilon_\mathrm{min} > 0$ and therefore collision avoidance (cf. \cref{thm:ca_point}). This property can directly be verified by inserting \cref{eq:scaling_kcf} in \cref{eq:def_epsilon}, which results in $\epsilon=\epsilon_\mathrm{min}\mathrm{sgn}(\epsilon)$.
Moreover, upon inspection of \cref{eq:bound_x_point_eps_less,eq:bound_x_point_eps_greater}, it becomes clear that $\epsilon$ is also related to the minimal robot obstacle distance, albeit establishing a uniform bound is challenging due to the dependence on $\kcf$ in \cref{eq:bound_x_point_eps_less}.
The condition in \cref{eq:scaling_kcf} should be checked when the robot comes very close to an obstacle, i.e., $\norm{\dv} \leq \dmin$ with $0 < \dmin < \dmax$, and when we are in the critical case $R<0,S>0$.
Note that in case $R$ and $S$ are simultaneously close to zero, the robot is already close to the obstacle (cf. \cref{eq:barrier_ab}) and it might not be possible to guarantee a minimal distance with the desired avoidance direction $\bv$.
In this case, \cref{eq:scaling_kcf} can result in $\tilde{k}_\mathrm{CF}<0$, which is equivalent to switching the sign of $\bv$ and therefore changing the desired avoidance direction.
Thus, \cref{eq:scaling_kcf} ensures a uniform bound on $\epsilon$, which enables us to ensure collision avoidance without necessarily restricting the avoidance direction $\bv$ around the obstacle.\\
Note that in the virtual agents framework (\cref{sec:multi-agents}) we always have one initial condition with $S<0$ and one with $S>0$.
Notably, the case $S\leq 0$ is not critical (cf. \cref{lma:point_Sless0}) and hence independent of the modification \eqref{eq:scaling_kcf}, there exists at least one collision free agent.\\
Note that in \cref{thm:ca_point} we assumed that the robot is controlled by the \gls{cf} force only, i.e., no attractive force for goal convergence is applied. Additionally, we assumed a scenario with only one point obstacle as described in \cref{sec:preliminaries}. These limitations will be addressed in the following section.
%
\section{Collision Avoidance for Multiple Obstacles with Attractive Force }
In the following, we extend the results from the previous section, such that the collision avoidance guarantees remain valid in the presence of multiple obstacles (\cref{sec:multi_obstacles}), arbitrarily shaped point clouds (\cref{sec:point_clouds}) and when used in conjunction with an appropriately scaled goal force (\cref{sec:combined_force_analysis}).
This is done by first considering additional bounded disturbance forces $\norm{\zv(t)} \leq \zmax$ (\cref{sec:disturbance}).
Subsequently, we exploit this robustness property by interpreting the forces of additional obstacles or an additional attractive force as disturbances.
Note that when additional disturbances are taken into account, it can no longer be guaranteed that $\norm{\dot{\xv}(t)}$ will remain constant.
However, in this section, we assume that there exist bounds on the robot velocity $\dxmin \leq \norm{\dot{\xv}(t)} \leq \dxmax$, which also yields
\begin{equation}
  \cmin = \frac{\kcf}{\dxmax^2}  \leq c(t) = \frac{\kcf}{\norm{\dot{\xv}(t)}^2} \leq \frac{\kcf}{\dxmin^2} = \cmax. \label{eq:def_cminmax}
\end{equation}
Compliance with these bounds will be discussed in detail later (cf. \cref{lma:maxvel,lma:minvel}).\\
Please note that we analyze the collision avoidance properties of the \gls{cf} algorithm. Therefore, in the following we only investigate robot positions which are close to the obstacle, i.e., we consider a maximum robot obstacle distance $\norm{\dv} = \norm{\xv} \leq \xmax \leq \dmax$.
%
%
\subsection{Collision Avoidance under Disturbances}\label{sec:disturbance}
In the following, we show that the previous lemmas hold even when an additional bounded disturbance $\zv$ with $\norm{\zv} \leq \zmax$ perturbs the circular field force.
Therefore, we include the following changes to the previous equations
\begin{align}
  \Fm_\mathrm{CF,d}             & = \Fm_\mathrm{CF} + \zv                                                                                                         \\
  \dot{R}                       & = \kcf \frac{R S}{R^2+S^2} + \norm{\dot{\xv}}^2 + \xv \cdot \zv \label{eq:dot_r_disturbed}                                      \\
  \dot{S}                       & = -\kcf \frac{R^2}{R^2+S^2} + \left(\xv \times \zv \right) \cdot \bv \label{eq:dot_s_disturbed}                                 \\
  \diff{}{t} \norm{\dot{\xv}}^2 & = 2 \dot{\xv} \cdot \ddot{\xv} = 2 \dot{\xv} \cdot \zv                                                                          \\
  \dot{c}                       & = -2 \kcf \frac{\dot{\xv} \cdot \zv}{\norm{\dot{\xv}}^4} = -2c \frac{\dot{\xv} \cdot \zv}{\norm{\dot{\xv}}^2}. \label{eq:dot_c}
\end{align}
%
%
We start the analysis again with the simplest case, i.e., the robot is moving away from the obstacle.
\begin{lemma}\label{lma:disturbance_Rgeq0}
  Suppose $\zmax \leq \frac{\dxmin^2}{\xmax}$. Then, for any $R(0) \geq 0$, the dynamics in \cref{eq:dot_r_disturbed,eq:dot_s_disturbed} yield $R(t) \geq 0$ and $V_B(t) \leq V_B(0)$  for all $t \geq 0$.
\end{lemma}
The proof can be found in \cref{sec:proof_disturbance_Rgeq0}.\\
For the following lemma and theorem, we study $R(t)\leq 0, t\in[0,\tau]$ for some $\tau>0$, which implies
\begin{equation}
  \norm{\xv(t)} \leq \norm{\xv(0)}, \label{eq:bound_norm_x}
\end{equation}
where we used \cref{eq:dot_barrier_rs} with  $\dot{V}_B(t) \geq 0$ and $\norm{\xv}^2=\frac{1}{V_B}$.
%
\begin{lemma}\label{lma:disturbance_Sless0}
  Suppose $\zmax \leq \min\left(\frac{\kcf}{\norm{\xv(0)}}\frac{\dxmin^2}{\dxmax^2},\frac{\dxmin^2}{2 \norm{\xv(0)}}, -\frac{\dxmin S(0)}{4 \norm{\xv(0)}^2}\right)$.
  Then, for any $R(0)<0, S(0) < 0$, there exists a time $\tau>0$ such that $R(\tau)=0$ and for all $t\in[0,\tau]$:  $V_B(t) \leq \frac{4 \dxmax^2}{S(0)^2}$.
\end{lemma}
The proof can be found in \cref{sec:proof_disturbance_Sless0}.\\
As discussed before, $R\geq 0$ represents the case when the robot has already passed the obstacle and subsequently the guarantees from \cref{lma:disturbance_Rgeq0} apply.\\
%
%
The result in~\cref{lma:disturbance_Sless0} deteriorates for $|S(0)|$ arbitrary small. However, uniform bounds for $V_B(t)$ in the form $V_B(t)\leq k_2V_B(0)$ with $k_2>0$ can also be derived, compare \cref{sec:proof_remark}.\\
%
%
For the analysis of the critical case in the upper left quadrant with disturbances, we again use the definition of $\epsilon$ in \cref{eq:def_epsilon} as the deviation from the collision ray. Note that the derivative changes to
\begin{equation} \label{eq:dot_eps_disturbed}
  \dot{\epsilon}               = \dot{S} + c \dot{R} + \dot{c} R.
\end{equation}
\begin{theorem}\label{lma:disturbance_critical}
  Suppose \begin{equation} \label{eq:bound_cmax}
    \cmax<\begin{cases}
      \frac{\cmin^2+1}{\cmin} & \text{if } \cmin\geq1 \\
      2                       & \text{otherwise}
    \end{cases}.
  \end{equation}
  Then, for any $R(0)<0, S(0) > 0$ and $\epsilon(0)\neq 0$, there exists a time $\tau>0$ and a disturbance bound $z_{\max}>0$, such that $R(\tau)=0$ or $S(\tau)=0$ and for all $t\in[0,\tau]$:  $|\epsilon(t)|\geq\frac{|\epsilon(0)|}{2}$.
  Moreover, $\norm{\xv(t)}\geq \frac{|\epsilon(0)|}{2 \dxmax \max(\cmax,1)}$, $t\in[0,\tau]$, i.e., no collision occurs.
\end{theorem}
The proof can be found in \cref{sec:proof_disturbance_critical}.\\
Similar to the undisturbed case, \cref{lma:disturbance_Rgeq0,lma:disturbance_Sless0,lma:disturbance_critical} cover all possible initial conditions and therefore guarantee collision avoidance except for a small set of initial conditions as discussed in the following.
Note that given a uniform bound for $\frac{\dxmax}{\dxmin}$, we can choose a sufficiently small constant $\kcf>0$ based on \cref{eq:def_cminmax} such that condition \eqref{eq:bound_cmax} in \cref{lma:disturbance_critical} holds.
As discussed in the beginning of this section, the bounds on the disturbance depend on the initial conditions. In both cases of \cref{lma:disturbance_critical}, we can observe a dependency on $\epsilon(0)$ (cf. \cref{eq:bound_z_critical_dist_below,eq:bound_z_critical_dist_above1,eq:bound_z_critical_dist_above2} in \cref{sec:proof_disturbance_critical}), which can be intuitively described as the distance to the collision ray $S = cR$ (cf. \cref{lma:point_critical}).\\
As shown in \cref{lma:disturbance_Sless0} and in the proof in \cref{sec:proof_disturbance_critical}, all disturbance bounds $\zmax$ are scaled with $\frac{1}{\norm{\xv(0)}}$ or $\frac{|\epsilon(0)|}{\norm{\xv(0)}^2}$. Therefore, as we approach an obstacle (and hence get $\norm{\xv(0)} \rightarrow 0$ in the proofs of \cref{lma:disturbance_Sless0,lma:disturbance_critical}), the maximal allowed disturbance $z_{\max}$ increases (as long as we are not on the collision ray, i.e., $\epsilon(0)\neq 0$.).
A similar reasoning can be used for the disturbance bound $\zmax$ in \cref{lma:disturbance_Rgeq0}, which is scaled with $\frac{1}{\xmax}$. We can choose $\xmax \rightarrow 0$, if the robot starts close to the obstacle, which increases the maximal allowed disturbance $z_{\max}$. Note that in this case the robot is moving away from the obstacle, which makes this case less critical anyways.
The only exception is the bound in \cref{lma:disturbance_Sless0}, when $\zmax = -\frac{\dxmin S(0)}{4 \norm{\xv(0)}^2}$, where we have the additional dependency on $S(0)$. However, as noted in the discussion after \cref{lma:disturbance_Sless0}, we can also use a different proof resulting in $\zmax \leq \min\left(\frac{\kcf}{2 \norm{\xv(0)} \left(\tilde{c}^2 + 1\right)}, \frac{\dxmax^2 + \kcf \frac{\tilde{c}}{\tilde{c}^2 + 1}}{2 \norm{\xv(0)}}\right)$.
%
\subsection{Multiple Obstacles} \label{sec:multi_obstacles}
In a next step, we extend our problem to multiple point-like obstacles.
Note that the velocity of the robot remains constant under \gls{cf} forces from multiple obstacles, i.e., $\dxmax = \dxmin, \cmax = \cmin$ (cf. \cref{lma:constant_vel}).
We interpret the \gls{cf} force from one or multiple obstacles as a disturbance to the \gls{cf} force of the closest obstacle $i$.
From the previous derivation, we know that there will be no obstacle collision if we stay away from the collision ray in the $RS$ plane ($|\epsilon(t)|>0$), which in turn holds for initial conditions away from the collision ray and for sufficiently small disturbances.
Furthermore, the \gls{cf} force $\Fm_\mathrm{CF}$ is scaled by the reciprocal of the distance between robot and obstacle (cf. \cref{eq:cf_force,eq:cf_b_field}).
We make the reasonable assumption that the obstacles are not arbitrarily close to each other, i.e., there exists a uniform lower bound on the distance between any two obstacles, and hence close to any obstacle the disturbance $z$ due to the \gls{cf} forces of other obstacles is uniformly bounded by some factor $z_{\max}$.
Additionally, as already noted in the discussion in \cref{sec:disturbance}, we have $\norm{\xv(0)} \rightarrow 0$ as we approach an obstacle (or choose $\xmax \rightarrow 0$ in case of \cref{lma:disturbance_Rgeq0}).
Considering the discussion after Theorem~\ref{lma:disturbance_critical}, we can avoid collisions for almost all initial conditions ($\epsilon(0)\neq 0$) if $z_{\max}$ is upper bounded by a factor $\frac{1}{\norm{\xv(0)}}$ or $\frac{|\epsilon(0)|}{\norm{\xv(0)}^2}$.
Hence, as $\norm{\xv(0)}\rightarrow 0$ (we consider points close to the obstacle), the set of initial conditions for which a collision is possible approaches measure zero.
Thus, the collision avoidance property of the \gls{cf} force shown in the previous lemmas essentially remains valid in the presence of multiple point-like obstacles (for most initial conditions).
Note that the constructive changes proposed in \cref{eq:scaling_kcf} can also be used in the vicinity of multiple obstacles, where $\frac{\dmin}{2}$ should be smaller than the distances between adjacent obstacles (ensuring that \cref{eq:scaling_kcf} is only active for one obstacle at a time).
%
%
\subsection{Point Cloud Obstacles} \label{sec:point_clouds}
An additional important consideration are obstacles that are represented by a point cloud.
In this case, we do not only require avoidance of the point obstacles, but we also need to ensure that the robot does not pass between two adjacent points (as the ``real'' obstacle corresponds also to the space between the point cloud points).
This follows naturally if the points are close to each other and we can ensure a lower bound on the distance to the obstacle.
To this end, the magnetic field vectors $\bv$ of the individual point obstacles within a point cloud are defined in the same direction (cf. \cref{eq:cf_current}), i.e., the induced \gls{cf} forces of the obstacle points point in a similar direction.
This is also consistent with the experimental results from \cite{HaddadinBelAlb2011}, in which the authors observed that different magnetic field vectors of surfaces of the same obstacle lead to oscillations.
In \cref{sec:simulations}, we illustrate the effect of the \gls{cf} force from point cloud obstacles by considering the critical case (\cref{lma:disturbance_critical}).
Therein, we also demonstrate empirically how the previously derived results ensure collision avoidance for most initial conditions.
%
%
\subsection{Attractive Potential Field Force} \label{sec:combined_force_analysis}
In this section, we analyze the combination of the \gls{cf} obstacle avoidance force with an additional attractive potential field, which leads to the following control law when using point mass dynamics for the robot
\begin{equation} \label{eq:combined_dynamics}
  \ddot{\xv} = \Fm_\mathrm{CF} + k_\mathrm{VLC} \Fm_\mathrm{VLC}.
\end{equation}
The attractive force should guide the robot to its goal position $\xv_\mathrm{g}\in\mathbb{R}^3$ and needs to ensure bounds on the robot velocity $\dxmin \leq \norm{\dot{\xv}} \leq \dxmax$ to guarantee the validity of the assumptions in \cref{lma:disturbance_Rgeq0,lma:disturbance_Sless0,lma:disturbance_critical}.
As in \cite{BeckerLilMulHad2021}, $\norm{\dot{\xv}} \leq \dxmax$ is achieved with the \gls{vlc} from \cite{Khatib1986}:
\begin{equation}
  \Fm_\mathrm{VLC} = -k_\mathrm{v}(\dot{\xv} - \nu \vv_\mathrm{d}),\text{ with } \vv_\mathrm{d}   = \frac{k_\mathrm{p}}{k_\mathrm{v}}(\xv_\mathrm{g} -\xv).  \label{eq:vlc_force}
\end{equation}
Here, $\vv_\mathrm{d}$ is an artificial desired velocity with the position gain $k_\mathrm{p}>0$ and the velocity gain $k_\mathrm{v}>0$.
The factor
\begin{equation} \label{eq:vlc_nu}
  \nu = \mathrm{min} \left( 1, \frac{\dxmax}{\norm{\vv_\mathrm{d}}} \right)
\end{equation}
ensures that the velocity magnitude does not exceed a specified limit $\dxmax$ as shown in the following \cref{lma:maxvel}.
%
\begin{proposition}\label{lma:maxvel}
  Suppose that $\norm{\dot{\xv}(0)}\leq \dxmax$, then the dynamics from \cref{eq:combined_dynamics} and \cref{eq:cf_current,eq:cf_b_field,eq:cf_force,eq:cf_force_sum,eq:vlc_force,eq:vlc_nu} yield $\norm{\dot{\xv}(t)} \leq \dxmax$ for all $t \geq 0$.
\end{proposition}
\begin{proof}
  Given that $\norm{\dot{\xv}(0)} \leq \dxmax$, it suffices to show that $\norm{\dot{\xv}}$ is non-increasing if $\norm{\dot{\xv}}=\dxmax$.
  Hence, we consider $\norm{\dot{\xv}}=\dxmax$ and a case distinction based on $\vv_\mathrm{d}(\xv)$ in \cref{eq:vlc_force}:\\
  \emph{Case i)}: $\norm{\vv_\mathrm{d}} < \dxmax$: Note that $\norm{\vv_\mathrm{d}} < \dxmax$ yields $\nu = 1$ (cf. \cref{eq:vlc_nu}).
  Using \cref{eq:vlc_force} and the result from \cref{lma:constant_vel}, i.e., $\dot{\xv} \cdot \Fm_\mathrm{CF} = 0$, we get
  \begin{align*}
    \diff{}{t}\frac{\norm{\dot{\xv}}^2}{2} & = \dot{\xv}\cdot \ddot{\xv}= \dot{\xv}\cdot \left(\Fm_\mathrm{CF} + k_\mathrm{VLC} \Fm_\mathrm{VLC}\right)                             \\
                                           & = k_\mathrm{VLC} \dot{\xv}\cdot \Fm_\mathrm{VLC} = - k_\mathrm{VLC} k_\mathrm{v} \dot{\xv}\cdot\left(\dot{\xv} - \vv_\mathrm{d}\right) \\
                                           & \leq k_\mathrm{VLC} k_\mathrm{v}\left(\norm{\dot{\xv}} \norm{\vv_\mathrm{d}}-\norm{\dot{\xv}}^2\right)                                 \\
                                           & = k_\mathrm{VLC} k_\mathrm{v}\dxmax\left(\norm{\vv_\mathrm{d}}-\dxmax\right)
    <0.
  \end{align*}
  \emph{Case ii)}: $\norm{\vv_\mathrm{d}} \geq \dxmax$: In this case we have $\nu=\frac{\dxmax}{\norm{\vv_\mathrm{d}}}$ (cf. \cref{eq:vlc_nu}) and the derivative is given by
  \begin{align*}
    \diff{}{t}\frac{\norm{\dot{\xv}}^2}{2} & = k_\mathrm{VLC} \dot{\xv}\cdot \Fm_\mathrm{VLC}                                                                                                             \\
                                           & = k_\mathrm{VLC} k_\mathrm{v} \dot{\xv}\cdot\left( \frac{\dxmax}{\norm{\vv_\mathrm{d}}}\vv_\mathrm{d} - \dot{\xv}\right)                                     \\
                                           & \leq k_\mathrm{VLC} k_\mathrm{v} \left( \frac{\dxmax}{\norm{\vv_\mathrm{d}}} \norm{\dot{\xv}} \norm{\vv_\mathrm{d}} - \norm{\dot{\xv}}^2\right)= 0. \qedhere
  \end{align*}
\end{proof}
In order to further ensure a minimum robot velocity, we define the goal force scaling factor as follows
\begin{equation}\label{eq:k_vlc}
  k_\mathrm{VLC} = \begin{cases}
    0 & \text{if }\dot{\xv} \cdot \Fm_\mathrm{VLC} \leq 0 \land \norm{\dot{\xv}} \leq \dxmin \\
      & \land \norm{\xv_\mathrm{g} - \xv} > \xi                                              \\
    1 & \text{otherwise}
  \end{cases}.
\end{equation}
This choice of the scaling factor ensures that the robot velocity does not drop below a minimum velocity  ($\norm{\dot{\xv}} \leq \dxmin$), unless the robot is close to the goal ($\norm{\xv_\mathrm{g} - \xv} \leq \xi$).
Note that this definition of $k_\mathrm{VLC}$ leads to a discontinuous steering force that could be avoided by implementing a smooth approximation of \cref{eq:k_vlc}. However, for practical application, the \gls{cfp} planner can also be used as a reference input to an appropriate velocity controller so a redefinition is not necessarily required (cf. \cite{BeckerLilMulHad2021,BeckerCasHatLilHadMul2023}).
%
\begin{proposition}\label{lma:minvel}
  Suppose that $\norm{\xv_\mathrm{g} - \xv(t)} \geq \xi$ for all $t\in[0,\tau]$ with some $\tau>0$ and $\norm{\dot{\xv}(0)} >\dxmin $. Then, the dynamics from \cref{eq:combined_dynamics} and $k_\mathrm{VLC}$ from \cref{eq:k_vlc} yield $\norm{\dot{\xv}(t)} \geq \dxmin$ for all $t\in[0,\tau]$.
\end{proposition}
\begin{proof}
  To show that $\norm{\dot{\xv}(0)} >\dxmin$ holds recursively it suffices to show that $\norm{\dot{\xv}}$ is non-decreasing if $\norm{\dot{\xv}}=\dxmin$.
  Thus, suppose that $\dot{\xv} \cdot \Fm_\mathrm{VLC} \leq 0$ and $\norm{\dot{\xv}}=\dxmin$, then $k_\mathrm{VLC} = 0$ and therefore $\diff{}{t}\frac{\norm{\dot{\xv}}^2}{2} = k_\mathrm{VLC} \dot{\xv}\cdot \Fm_\mathrm{VLC} = 0$.
  Note that $\dot{\xv}\cdot \Fm_\mathrm{VLC} > 0$ implies $\diff{}{t}\frac{\norm{\dot{\xv}}^2}{2}>0$.
\end{proof}
\cref{lma:minvel} ensures that a minimum velocity is kept as long as the robot is outside a ball of radius $\xi$ around the goal position. If $\xi$ is suitably defined (small enough) and assuming that there are no obstacles arbitrarily close to the goal $\xv_\mathrm{g}$, the \gls{cf} force becomes inactive and hence a lower bound on the velocity is no longer needed.\\
The results from \cref{lma:maxvel,lma:minvel} and \cref{sec:disturbance} allow us to use the same argument as in \cref{sec:multi_obstacles}.
The \gls{vlc} force is interpreted as a disturbance to the \gls{cf} force of the closest obstacle, which grows the closer the robot moves to the obstacle, while the \gls{vlc} force is uniformly bounded given that $\norm{\dot{\xv}}$ is not arbitrarily large.
%
%
\section{Goal Convergence} \label{sec:goal_convergence}
No collisions are a prerequisite to facilitate goal convergence and thus for the following theorem.
Using the results from \cref{lma:disturbance_Rgeq0,lma:disturbance_Sless0,lma:disturbance_critical} and the argumentation from \cref{sec:multi_obstacles,sec:point_clouds}, we can conclude that for almost all initial conditions, no robot obstacle collision occurs.
The following theorem studies the global convergence to the goal position of a robot, that is controlled by the combined steering force resulting in the dynamics from \cref{eq:combined_dynamics}.
\begin{theorem} \label{lma:goal_convergence}
  Suppose there exists a time $\tau>0$, such that $k_\mathrm{VLC}=1$ for all $t \geq \tau$ with $k_\mathrm{VLC}$ from \cref{eq:k_vlc}. Then, the equilibrium $\xv = \xv_\mathrm{g}$, $\dot{\xv}= 0$ is globally attractive for almost\footnote{For almost all initial conditions, other than those as specified in \cref{lma:disturbance_critical} assuming $\norm{F_\mathrm{VLC}} < \zmax$.} all initial conditions.
\end{theorem}
\begin{proof}
  First, we define an energy-inspired Lyapunov function, which is an extended version of the function used in \cite{SinghSteWen1996}
  \begin{equation}\label{eq:lyapunov}
    V(\xv, \dot{\xv}) = \frac{1}{2}\dot{\xv}^T \dot{\xv} + U(\xv),
  \end{equation}
  where $U(\xv)$ is a potential function given by a Huber loss
  \begin{align*}
    U & (\xv)=                      \\
      & \begin{cases}
      \frac{1}{2}k_\mathrm{p} \norm{\xv - \xv_\mathrm{g}}^2                                              & \text{if } \norm{\xv-\xv_{\mathrm{g}} } < \frac{k_{\mathrm{v}}\dxmax}{k_{\mathrm{p}}} \\%
      k_{\mathrm{v}}\dxmax\norm{\xv-\xv_{\mathrm{g}}}-\dfrac{k_{\mathrm{v}}^2\dxmax^2}{2 k_{\mathrm{p}}} & \text{otherwise}
    \end{cases}.
  \end{align*}
  Note that the attractive part of our goal force corresponds to the negative gradient of this potential field, i.e.,
  \begin{align}
    \nabla U(\xv)                                       & =\begin{cases}
      k_{\mathrm{p}}(\xv-\xv_{\mathrm{g}})                                         & \text{if }\norm{\xv-\xv_{\mathrm{g}} } < \frac{k_{\mathrm{v}}}{k_{\mathrm{p}}}\dxmax \\
      k_\mathrm{v} \dxmax \frac{\xv - \xv_\mathrm{g}}{\norm{\xv - \xv_\mathrm{g}}} & \text{otherwise}
    \end{cases}                  \label{eq:grad_U} \\
    \overset{\eqref{eq:vlc_force}-\eqref{eq:vlc_nu}}{=} & - \Fm_\mathrm{VLC} - k_\mathrm{v} \dot{\xv}. \nonumber
  \end{align}
  For all $t\geq \tau$, i.e., $k_\mathrm{VLC} = 1$, the combined dynamics in \cref{eq:combined_dynamics} yield
  \begin{equation*}
    \ddot{\xv} = \Fm_\mathrm{CF} + k_\mathrm{VLC} \Fm_\mathrm{VLC} = \Fm_\mathrm{CF} - \nabla U(\xv)- k_v \dot{\xv}.
  \end{equation*}
  Thus, the derivative of the Lyapunov function becomes
  \begin{align*}
    \diff{}{t}V(\xv, \dot{\xv}) & = \dot{\xv}^T \ddot{\xv} + \nabla U(\xv)^T \dot{\xv}                                                     \\
                                & = \dot{\xv}^T \left( \Fm_\mathrm{CF}  - \nabla U(\xv)- k_v \dot{\xv} \right) + \nabla U(\xv)^T \dot{\xv} \\
                                & = - k_v \dot{\xv}^T \dot{\xv},
  \end{align*}
  where for the last equality we used the fact that $\dot{\xv}^T \Fm_\mathrm{CF} = 0$ (cf. \cref{eq:dotx_Fcf_0}).
  The invariance principle from Krasovskii and LaSalle ensures that we converge to an equilibrium, i.e., $\dot{\xv}=0$.
  Note that $\dot{\xv}\equiv 0$ implies $\ddot{\xv}=-\nabla U\stackrel{!}{=}0$, which only holds for $\xv = \xv_\mathrm{g}$ using \cref{eq:grad_U}.
  This shows that the equilibrium $\xv = \xv_\mathrm{g}$, $\dot{\xv}= 0$ is the largest invariant subset of $\{\xv, \dot{\xv} \, | \, \dot{V}(\xv, \dot{\xv}) = 0 \}$. Together with the fact that $V$ is radially unbounded, this ensures that $\xv = \xv_\mathrm{g}$, $\dot{\xv}= 0$ is globally attractive for almost all initial conditions.
\end{proof}
Intuitively, the previous theorem states that the robot will converge to a goal position $\xv_\mathrm{g}$, when there exists a time after which the attractive force is always active.
The deactivation of the attractive force only happens if it would decrease the robot velocity below a specified minimum ($\norm{\dot{\xv}} \leq \dxmin$, $\dot{\xv} \cdot \Fm_\mathrm{VLC} \leq 0$).
In these cases, the \gls{cf} force will guide the robot along the surface of the obstacle (cf. \cite{AtakaLamAlt2018,HaddadinBelAlb2011}).
Assuming obstacles that are not infinitely large or blocking all traversable paths to the goal position, at some point the robot will move again in the direction of the goal position, which will reactivate the attractive force and the robot will eventually leave the influence of the obstacle.
Therefore, we can presume that it is not possible for $k_\mathrm{VLC}=0$ to hold permanently.
Nonetheless, there exists scenarios in which the robot is trapped in a limit cycle around obstacles and the attractive force is alternately switched on and off.
Limit cycles are a known drawback of the \gls{cf} motion planner that was also already reported in \cite{HaddadinBelAlb2011} and it is possible to construct scenarios, e.g., elaborate maze-like environments, in which goal convergence is not achieved.
Thus, we cannot a-priori guarantee that a time $\tau < \infty$ as required in \cref{lma:goal_convergence} always exists.
Nevertheless, with a suitable choice of $\kcf$, $k_\mathrm{p}$ and $k_\mathrm{v}$, goal convergence is achieved even in complex environments as shown in our simulation example in the following.
%
%
\section{Simulation} \label{sec:simulations}
In this section, we demonstrate collision avoidance and goal convergence of the \gls{cfp} planner in two complex environments with multiple nonconvex point cloud obstacles. The first environment in \cref{subfig:sim_2d} highlights the theoretical results of this paper. In particular, we show multiple simulated trajectories and situations with the critical conditions from \cref{lma:disturbance_critical}, i.e., $R<0$, $S>0$.
The second environment in \cref{subfig:sim_3d_top} is used to demonstrate the practical capabilities of the planner beyond the theoretical guarantees in a \textit{3D} environment with multiple \textit{dynamic} obstacles and \textit{noisy} sensor measurements.
Additionally, we compare the \gls{cfp} planner against the \gls{apf} planner from \cite{Khatib1986} in both environments and provide evaluations of several performance criteria in \cref{tab:performance}.
Additional details regarding the simulations and videos of all simulations are provided online in \cite{datarepo2023}.\\
\begin{figure*}%
  \centering
  \begin{subfigure}{.49\textwidth}
    \centering
    \includegraphics[height=0.24\textheight]{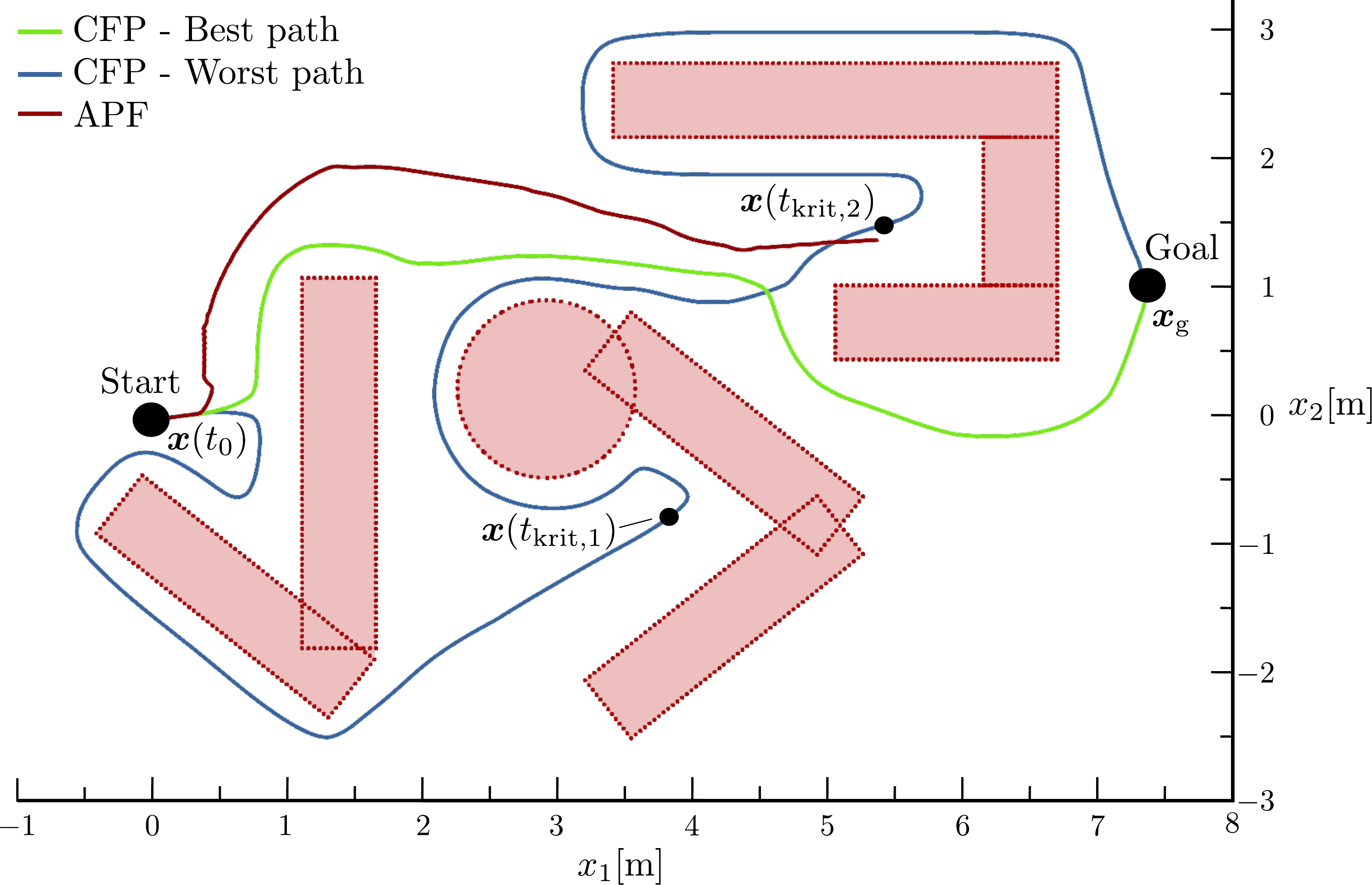}%
    \caption{2D simulation}%
    \label{subfig:sim_2d}%
  \end{subfigure}\hfill%
  \begin{subfigure}{.49\textwidth}
    \centering
    \includegraphics[height=0.24\textheight]{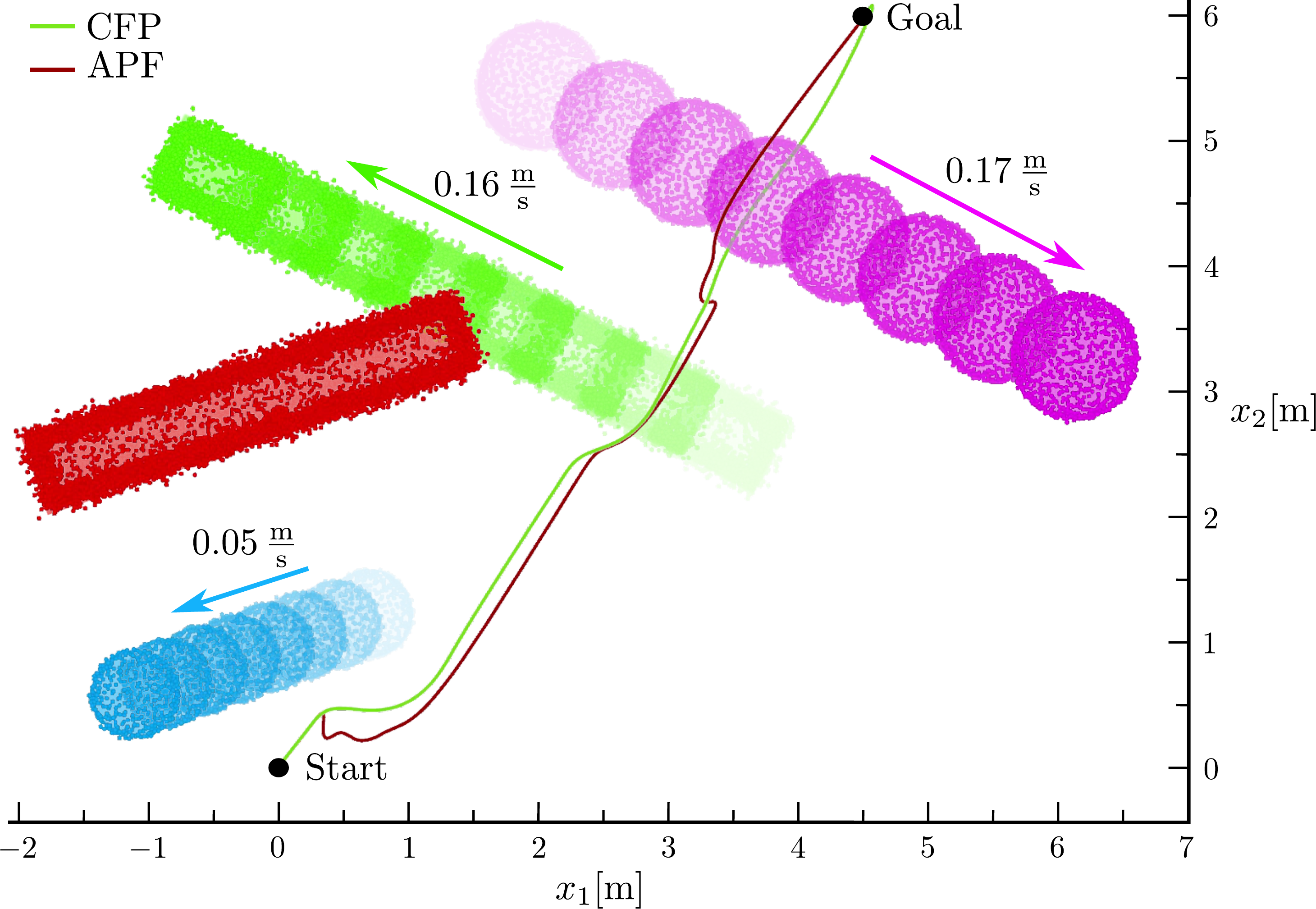}%
    \caption{3D simulation - top view}%
    \label{subfig:sim_3d_top}%
  \end{subfigure}
  \caption{Exemplary paths in two challenging environments using the \gls{cfp} and the \gls{apf} the motion planner. The most critical path of the \gls{cfp} is depicted in blue. The best paths, chosen by the virtual agents framework, are depicted in green and the paths of the \gls{apf} planner are red. Points of interest of the robot trajectory are shown in black. The left figure shows the results of the static \gls{2d} simulation. The right figure depicts the top view of the simulation in a \gls{3d} environment with dynamic obstacles and noisy point cloud data.}
  \label{fig:simulations}
\end{figure*}
All simulations were conducted 10 times due to the random noise in the \gls{3d} environment and to provide meaningful average computation times.
We discretized the simulation with an Euler using a frequency of $\SI{1}{\kilo\hertz}$ and used a computer with an Intel Core i9-9880H CPU, 2.30\,GHz and 16\,GB of memory, a C++ implementation and the Robot Operating System (ROS) \cite{QuigleyConGerFau2009}.\\
\Cref{subfig:sim_2d} shows the simulation of the \gls{cfp} motion planner in a complex \gls{2d} environment with multiple nonconvex point cloud obstacles.
In total, the virtual agents framework (cf. \cref{sec:multi-agents}) simulated 8 different trajectories, all of which successfully reached the goal without any collision.
In the figure, only the optimal robot trajectory (green) and the simulated path with most critical situations (blue), i.e., in which the robot encounters the critical conditions from \cref{lma:disturbance_critical} multiple times, are shown.
Other possible paths that were simulated by the virtual agents framework can be seen in the videos \cite{datarepo2023}.
The choice of the best trajectory was determined in both scenarios by the same cost function, using path length and minimal obstacle distance as the main evaluation criteria as described in detail in \cite{BeckerLilMulHad2021}.
The occurrence of the critical conditions from \cref{lma:disturbance_critical}, i.e., $R<0$, $S>0$, are highlighted in \cref{subfig:sim_2d} in the blue trajectory.
Here, the robot velocity points in a direction opposing the intended movement around the obstacle, which is defined by the magnetic field vector. Therefore, the \gls{cf} force first needs to change the direction of the robot to evade the obstacle in the desired way. Note that $R$ and $S$ are always defined by the closest obstacle point.
The additional scaling of $\kcf$ from \cref{eq:scaling_kcf} was never invoked because the minimal robot obstacle distance was always greater than $\dmin = \SI{0.1}{\meter}$.
In contrast, the \gls{apf} planner is not able to reach the goal in this scenario because it got stuck in a local minimum in front of the last obstacle.\\
The \gls{3d} environment is particularly challenging because the velocity of the green obstacle was defined such that it forms a barrier with the red obstacle just when the robot would try to pass between them. As can be seen in the video, the virtual agents framework first chose a trajectory in between the red and the green obstacle and only changed to a trajectory passing the green obstacle on the right side when the passage between them became to small.
In this scenario, the maximum number of simultaneous virtual agents was 22.
The \gls{apf} planner was able to reach the goal with only slightly longer path length than the \gls{cfp}. The small difference results from the initial oscillating behavior of the \gls{apf} planner when the robot is repelled from the first obstacle, which delayed the robot so that the green obstacle only marginally influenced the trajectory of the robot and made the avoidance behavior much easier.
Nevertheless, the resulting trajectory of the \gls{cfp} is shorter and takes less time than the \gls{apf} trajectory.
\begin{table}
  \centering
  \caption{Performance Evaluation}
  \begin{subtable}{\columnwidth}
    \captionsetup{font=scriptsize}
    \centering
    \caption{Planner Comparison}\label{tab:plan_comp}
    \vspace{-1ex}
    \resizebox{0.98\columnwidth}{!}{%
      \begin{tabular}{llcccc} \toprule
        {Env.}              & {Planner} & {Length [m]}  & {Duration [s]} & {Min. Dist. [m]} & {Comp. Time [\si{\micro\second}]} \\ \midrule

        \multirow{2}{*}{2D} & \gls{apf} & Failed        & Failed         & 0.29             & \num{24.87(1746)}                 \\
                            & \gls{cfp} & 8.50          & 30.0           & 0.18             & \num{33.33(3270)}                 \\ \midrule

        \multirow{2}{*}{3D} & \gls{apf} & \num{8.52(1)} & \num{27.5(1)}  & \num{0.45(2)}    & \num{46.83(17425)}                \\
                            & \gls{cfp} & \num{8.43(7)} & \num{24.7(2)}  & \num{0.24(3)}    & \num{56.13(21275)}                \\ \bottomrule
      \end{tabular}%
    }
  \end{subtable}
  \begin{subtable}{\columnwidth}
    \captionsetup{font=scriptsize}
    \centering
    \vspace{2ex}
    \caption{Computation time virtual agents}\label{tab:pred_time}
    \vspace{-1ex}
    \resizebox{0.85\columnwidth}{!}{%
      \begin{tabular}{llcc}\toprule
        \multirow{2}{*}{Env.} & \multirow{2}{*}{Virtual Agent} & \multicolumn{2}{c}{Prediction Time [ms]}                              \\
                              &                                & \SI{1}{ms} Discretization                & \SI{10}{ms} Discretization \\ \midrule
        \multirow{2}{*}{2D}   & Best agent                     & \num{131.20(618)}                        & \num{15.2(35)}             \\
                              & Avg. agent                     & \num{253.23(5564)}                       & \num{30.5(123)}            \\ \midrule
        \multirow{2}{*}{3D}   & Best agent                     & \num{109.5(71)}                          & \num{11.3(25)}             \\
                              & Avg. agent                     & \num{120.9(1197)}                        & \num{14.0(203)}            \\ \bottomrule
      \end{tabular}%
    }
  \end{subtable}
  \label{tab:performance}
\end{table}
More quantitative performance criteria are shown in \cref{tab:performance}, including the minimal robot obstacle distance along the whole trajectory and the calculation time of a new control signal, which is marginally longer for the \gls{cfp} but still remains well below the desired sampling time of $\SI{1}{\milli\second}$.
Additionally, \cref{tab:pred_time} shows the calculation times of the predicted trajectories of the best agent and the average computation time of all agents with two different discretization times $\SI{1}{\milli\second}$ and $\SI{10}{\milli\second}$.
Note that it is possible to adapt the discretization to the respective task and environment without compromising the resulting avoidance behavior, i.e., we can use a coarser discretization for the prediction in order to further reduce the computation time, if necessary. This is relevant, for instance, in particularly complex environments, as we demonstrate in an additional video in the accompanying repository where we used a discretization of \SI{500}{\milli\second}.
Note that the longer prediction time in the \gls{2d} scenario results from the fact that a lower maximum velocity was used and the resulting trajectories are also longer on average, which in turn results in more average trajectory points ($\num{58496(18543)}$ trajectory points in the \gls{2d} scenario versus $\num{14964(9847)}$ trajectory points in the \gls{3d} scenario with \SI{1}{\kilo\hertz} frequency).\\
Overall, we see that the \gls{cfp} planner can cope with complex nonconvex environments and achieves simultaneous goal converge and obstacle avoidance, as expected from the derived theory.
%
%
\section{Conclusion}
We presented a rigorous mathematical analysis of the complete motion planning algorithm in \cite{BeckerLilMulHad2021}, resulting in formal guarantees for collision avoidance and goal convergence in planar environments.
In contrast to previous approaches, we did not consider isolated \gls{cf} forces, but studied the entire motion planner consisting of \gls{cf} forces for collision avoidance and attractive \gls{vlc} forces for goal convergence.
The extension to bounded disturbances enabled us to also guarantee collision avoidance in environments with multiple point obstacles.
Additionally, these considerations were qualitatively extended to point cloud obstacles.
In our analysis, we also consistently considered the unique ability of the \gls{cfp} planner to take different paths around an obstacle (compared to other \gls{cf} approaches).
We found tight conditions (set of measure zero) under which a collision is possible.
Since we applied the \gls{cf}-planner in a virtual agents framework, we wanted to ensure collision avoidance for all predicted trajectories.
This was done by verifying the collision conditions whenever the robot comes into the vicinity of an obstacle and adjusting the scaling factors where necessary (cf. \cref{eq:scaling_kcf}).
Furthermore, we extended the goal convergence analyses of previous approaches by replacing the simple damped attractive potential field with the \gls{vlc} force and ensured uniformly lower and upper bounded velocities of the robot.
A caveat of our analysis is that the reasoning for collision avoidance of point cloud obstacles is only of qualitative nature. This continues to be our focus in current research. Moreover, as discussed in \cref{fn:3d_analysis}, the presented analysis is also valid in a 3D setting for special choices of the magnetic field vector. We plan to derive guarantees for collision avoidance in 3D environments for general choices of  the magnetic field vector, for which a suitably adapted auxiliary system will be required.
%
%
\appendix
\subsection{Auxiliary Lemma}
\begin{lemma}\label{lma:bound_ab-term}
  Given constants $R<0$, $S>0$, $c\geq c_{\min}>0$, with $S+cR>0$, we have
  \begin{align}\label{eq:bound_ab-term}
    \dfrac{RS}{R^2+S^2}\geq \begin{cases}
      - \frac{\cmin}{\cmin^2+1} & \text{if } \cmin\geq 1 \\
      - 0.5                     & \text{if } \cmin < 1.
    \end{cases}
  \end{align}
\end{lemma}
The proof can be found in the repository \cite{datarepo2023}.
%
%
\subsection{Modified Bounds for Lemma 3 and Lemma 6}\label{sec:proof_remark}
In the following, we extend the results in \cref{lma:point_Sless0,lma:disturbance_Sless0} to provide uniform bounds also in case $|S(0)|$ small.
We directly consider the setting with disturbances from \cref{lma:disturbance_Sless0}, which contains the setting in \cref{lma:point_Sless0} as a special case.\\
\emph{Case a)}: First consider $|S(0)| \geq \tilde{c}|R(0)|$ with some $\tilde{c} > 0$, then \cref{lma:disturbance_Sless0} ensures a uniform linear bound for $V_B$, i.e.,
\begin{align*}
  V_B(t) & \leq \frac{4\dxmax^2}{S(0)^2} = \frac{8\dxmax^2}{S(0)^2+S(0)^2}                  \nonumber                                                   \\
         & \leq \frac{8\dxmax^2}{S(0)^2+\tilde{c}^2 R(0)^2} \leq \frac{8\dxmax^2}{\max\left(1,\tilde{c}^2 \right) \left(S(0)^2+R(0)^2\right)} \nonumber \\
         & \leq V_B(0)\frac{8\dxmax^2}{\dxmin^2 \max\left(1,\tilde{c}^2 \right)} =V_B(0)\frac{8\dxmax^2}{\dxmin^2\max\left(1,\tilde{c}^2 \right)}.
\end{align*}
\emph{Case b)}: Suppose without loss of generality that $|S(t)| < \tilde{c}|R(t)|$ for all $t\geq0$ (otherwise the bounds of \emph{Case a)} can still be applied with $|S(t)|\geq|S(0)|$).
Using \cref{eq:dot_s_disturbed}, we obtain
\begin{align}\label{eq:dot_s_disturbed_remark}
  \dot{S}(t) & \leq - \kcf\frac{R(t)^2}{R(t)^2+\tilde{c}^2 R(t)^2} + \norm{\xv(0)} \zmax \nonumber \\
             & = -\frac{\kcf}{\tilde{c}^2 + 1}+ \norm{\xv(0)} \zmax.
\end{align}
In the critical case, $R$ moves arbitrarily close to $0$ while $S$ stays arbitrarily close to $0$, i.e., there exists a time $\tau>0$ such that $R(\tau)=\frac{R(0)}{2}$ (otherwise from \cref{eq:barrier_ab} we directly have $V_B(t)\leq 4 \frac{\dxmax^2}{\dxmin^2} V_B(0)$). Suppose the disturbance bound satisfies $\zmax \leq \frac{\kcf}{2 \norm{\xv(0)} \left(\tilde{c}^2 + 1\right)}$, then integration of \cref{eq:dot_s_disturbed_remark} yields $S(t)\leq S(\tau)\leq S(0)-\frac{\kcf}{2\left(\tilde{c}^2 + 1\right)}\tau \leq-\frac{\kcf}{2\left(\tilde{c}^2 + 1\right)}\tau$ for all $t\geq \tau$.
Finally, suppose the disturbance bound additionally satisfies $\zmax \leq \frac{\dxmax^2 + \kcf \frac{\tilde{c}}{\tilde{c}^2 + 1}}{2 \norm{\xv(0)}}$. Then, by showing a uniform bound $\frac{1}{\tau^2} \leq k_1 V_B(0)$, $k_1>0$ and using $|R(t)|\geq\frac{|R(0)|}{2}$, $t\in[0,\tau]$, one can show a uniform bound on $V_B(t)$, i.e., $V_B(t)\leq k_2V_B(0)$ with some $k_2>0$.
%
%
\subsection{Proof of Lemma 5} \label{sec:proof_disturbance_Rgeq0}
\begin{proof}
  In case $R=0$, $\zmax \leq \frac{\dxmin^2}{\xmax}$ implies
  \begin{align*}
    \dot{R} & = \norm{\dot{\xv}}^2  + \xv \cdot \zv \geq \norm{\dot{\xv}}^2 - \norm{\xv} \norm{\zv} \\
            & \geq \dxmin^2 - \xmax \zmax \geq 0.
  \end{align*}
  Hence, we have $R(t) \geq 0$ for all $t \geq 0$.
  Furthermore, using \cref{eq:dot_barrier_rs}, we have $\dot{V}_B(t) \leq 0$, which yields
  \begin{equation*}
    V_B(t) \leq V_B(0),\quad t\geq 0.\quad \qedhere
  \end{equation*}
\end{proof}
\subsection{Proof of Lemma 6}\label{sec:proof_disturbance_Sless0}
\begin{proof}
  \textbf{Part I.}
  Define
  \begin{equation}
    \tau_1:= \inf_{t\geq 0,R(t)\geq-\delta \lor S(t)=0} t \label{eq:def_tau_1_lemma_6}
  \end{equation}
  with $\delta = \norm{\xv(0)}\dxmin$.
  Consider $t \in [0,\tau_1]$, which implies $R(t)< -\delta$ and that Inequality \eqref{eq:bound_norm_x} holds.\\
  In the following, we first show that $\tau_1$ is finite and then derive an upper bound for $S(t)$.\\
  Consider the disturbance bound
  \begin{equation}
    \zmax \leq \frac{\dxmin^2}{2 \norm{\xv(0)}}. \label{eq:bound_z_dot_r}
  \end{equation}
  Then, \cref{eq:dot_r_disturbed} yields
  \begin{align}
    \dot{R}(t) & = \kcf \frac{R S}{R^2+S^2} + \norm{\dot{\xv}}^2 + \xv \cdot \zv \nonumber                                                        \\
               & \geq \dxmin^2 - \norm{\xv(0)}\zmax \overset{\eqref{eq:bound_z_dot_r}}{\geq} \frac{\dxmin^2}{2} > 0 \label{eq:bound_dotR_noncrit}
  \end{align}
  for all $t\in [0, \tau_1]$.\\
  Suppose for contradiction that $\tau_1 > \tau_{1,\max}$ with
  \begin{equation}
    \tau_{1,\max} = -\frac{2 \norm{\xv(0)}}{\dxmin} -\frac{2 R(0)}{\dxmin^2}. \label{eq:def_tau1max_lemma_6}
  \end{equation}
  Then, integration of \cref{eq:bound_dotR_noncrit} yields
  \begin{align*}
    R(\tau_1) & > R(0) + \frac{\dxmin^2}{2} \tau_1                                                                                                               \\
              & \overset{\eqref{eq:def_tau1max_lemma_6}}{\geq} R(0) + \frac{\dxmin^2}{2} \left( -\frac{2 \norm{\xv(0)}}{\dxmin} -\frac{2 R(0)}{\dxmin^2} \right) \\
              & = -\norm{\xv(0)}\dxmin = -\delta,
  \end{align*}
  which contradicts $R(t)\leq -\delta$, $t\in[0, \tau_1]$. Hence, by contradiction $\tau_1$ exists with the upper bound $\tau_1 \leq \tau_{1,\max}$.\\
  Consider the additional disturbance bound
  \begin{equation}
    \zmax \leq \frac{\kcf}{\norm{\xv(0)}}\frac{\dxmin^2}{\dxmax^2}. \label{eq:bound_z_dot_s}
  \end{equation}
  Given $R(t) \leq - \delta$ and recalling that $R^2+S^2 = \norm{\xv}^2\norm{\dot{\xv}}^2$, \cref{eq:dot_s_disturbed} yields
  \begin{align*}
    \dot{S}(t) & \leq  -\kcf \frac{\delta^2}{\norm{\xv}^2\dxmax^2} + \norm{\xv} \zmax                                                                                                  \\
               & \overset{\eqref{eq:bound_norm_x}}{\leq} -\kcf \frac{\norm{\xv(0)}^2\dxmin^2}{\norm{\xv(0)}^2\dxmax^2} + \norm{\xv(0)}\zmax \overset{\eqref{eq:bound_z_dot_s}}{\leq} 0
  \end{align*}
  and therefore
  \begin{equation}
    S(t) \leq S(0) \quad \forall t\in[0,\tau_1]. \label{eq:bound_s_noncrit1}
  \end{equation}
  \textbf{Part II.}
  In the following, we study the time interval $t\in[\tau_1,\tau]$, with $\tau = \tau_1+\tau_2$, where
  \begin{equation}
    \tau_2:= \inf_{t\geq \tau_1,R(t)=0 \lor S(t)=0} t -\tau_1. \label{eq:def_tau_2_lemma_6}
  \end{equation}
  We first show that we leave the lower left quadrant in finite time and then deduce the upper bound for $S$ (and consequently $V_B$) until the time when we leave the quadrant.\\
  Also note that \cref{eq:def_tau_1_lemma_6,eq:def_tau_2_lemma_6} imply $R(t)\in[-\delta,0]$ for all $t\in[\tau_1,\tau]$ and suppose for contradiction that $\tau > \tau_1 + \tau_{2,\max}$ with
  \begin{equation}
    \tau_{2,\max} = \frac{2 \norm{\xv(0)}}{\dxmin}. \label{eq:def_delta_t_lemma_6}
  \end{equation}
  Then, integration of \cref{eq:bound_dotR_noncrit} yields
  \begin{align*}
    R(\tau) & > R(\tau_1) + \frac{\dxmin^2}{2} \tau_{2,\max} \geq -\delta + \frac{\dxmin^2}{2} \frac{2 \norm{\xv(0)}}{\dxmin} \\
            & \geq -\norm{\xv(0)}\dxmin + \norm{\xv(0)}\dxmin = 0,
  \end{align*}
  which contradicts $R(t)\leq 0$, $t\in[\tau_1,\tau]$. Hence, by contradiction we leave the lower left quadrant with $\tau \leq \tau_1 + \tau_{2,\max}$.\\
  Moreover, \cref{eq:dot_s_disturbed} yields
  \begin{equation}
    \dot{S}(t) = -\kcf \frac{R^2}{R^2+S^2} + \left(\xv \times \zv \right) \cdot \bv \overset{\eqref{eq:bound_norm_x}}{\leq} \norm{\xv(0)}\zmax \label{eq:bound_dotS_noncrit}
  \end{equation}
  for all $t\in[\tau_1,\tau]$.
  Consider the additional disturbance bound
  \begin{equation}
    \zmax \leq -\frac{\dxmin S(0)}{4 \norm{\xv(0)}^2}. \label{eq:bound_z_s_lemma_6}
  \end{equation}
  Then, integration of \cref{eq:bound_dotS_noncrit} yields
  \begin{align}
    S(t) & \leq S(\tau_1) + \norm{\xv(0)} \zmax \tau_{2,\max}                               \nonumber                                                                                      \\
         & \overset{\eqref{eq:def_delta_t_lemma_6},\eqref{eq:bound_z_s_lemma_6}}{\leq} S(0) - \norm{\xv(0)} \frac{\dxmin S(0)}{4 \norm{\xv(0)}^2} \frac{2 \norm{\xv(0)}}{\dxmin} \nonumber \\
         & \leq \frac{S(0)}{2} < 0 \label{eq:bound_s_noncrit2}
  \end{align}
  for all $t\in[\tau_1,\tau]$,  which also implies
  \begin{equation*}
    V_B(t) \overset{\eqref{eq:barrier_ab},\eqref{eq:bound_s_noncrit1},\eqref{eq:bound_s_noncrit2}}{\leq} \frac{4\dxmax^2}{S(0)^2}
  \end{equation*}
  for all $t\in[0,\tau]$.
  We would like to point out that we used $S(t)<0$ for all $t\in[0,\tau]$ in \cref{eq:bound_dotR_noncrit}, which holds using \cref{eq:bound_s_noncrit1,eq:bound_s_noncrit2}.
\end{proof}
\subsection{Proof of Theorem 2}\label{sec:proof_disturbance_critical}
\begin{proof}
  Analogously to \cref{lma:point_critical}, the following proof is split into three parts. First, we show that we leave the critical quadrant in finite time, i.e., there exists a constant $\tau_{\max}>0$, such that \begin{equation}
    \tau_{\max} \geq  \tau:= \inf_{t\geq 0, S(t)=0 \lor R(t)=0} t.
  \end{equation}
  Then, we show $|\epsilon(t)|\geq \frac{|\epsilon(0)|}{2}$, $t\in[0,\tau]$. Finally, we prove a lower bound on $\norm{\xv}$ for $\epsilon(0)\neq 0$.\\
  %
  \textbf{Part I.}
  In the following we distinguish two cases: $\epsilon(0)<0$, $\epsilon(0)>0$.
  We show that for both cases, there exists a time $\tau \leq \tau_\mathrm{max}$ such that $R(\tau)=0$ or $S(\tau)=0$.\\
  We would like to point out that in the following, we repeatedly use
  \begin{equation} \label{eq:eps_geq_eps0_disturbed}
    |\epsilon(t)| \geq \frac{|\epsilon(0)|}{2} > 0,
  \end{equation}
  for all $t\in[0,\tau]$, which will be established in Part~II.\\
  %
  \emph{Case i)}: $\epsilon(0) <0$:
  Note that \cref{eq:eps_geq_eps0_disturbed} and $\epsilon(0) <0$ imply $\epsilon(t) <0$ and hence $R(t)<0$ for all $t\in [0,\tau]$.\\
  Then, \cref{eq:eps_geq_eps0_disturbed} yields
  \begin{align*}
    \epsilon(t) & = S(t) + c(t)R(t) < 0,         \\
    S(t)^2      & < c^2R(t)^2 \leq \cmax^2R(t)^2
  \end{align*}
  for all $t \in [0,\tau]$.
  Together with \cref{eq:bound_norm_x}, we can rearrange \cref{eq:dot_s_disturbed} to get
  \begin{align}
    \dot{S} & < - \frac{\kcf}{1+\cmax^2} + \norm{\xv(0)} \zmax. \label{eq:bound_dotS_below_disturbed}
  \end{align}
  Consider the disturbance bound
  \begin{equation} \label{eq:bound_z_below_disturbed}
    \zmax \leq \frac{\kcf}{2 \norm{\xv(0)} \left(1+\cmax^2\right)}.
  \end{equation}
  For contradiction, suppose that $\tau>\tau_{\max,1d}$ with
  \begin{equation}
    \tau_{\max,1d} = \frac{2 S(0) \left(1+\cmax^2\right)}{\kcf}. \label{eq:tmax_below_disturbed}
  \end{equation}
  Then, the integration of \cref{eq:bound_dotS_below_disturbed} yields
  \begin{align*}
    S(\tau_{\max,1d}) & < S(0) + \left(- \frac{\kcf}{1+\cmax^2} + \norm{\xv(0)} \zmax \right)\tau_{\max,1d}           \\
                      & \leq S(0) - \frac{\kcf}{2\left(1+\cmax^2\right)} \tau_{\max,1d}                               \\
                      & = S(0) - \frac{\kcf}{2\left(1+\cmax^2\right)} \frac{2 S(0) \left(1+\cmax^2\right)}{\kcf} = 0,
  \end{align*}
  which contradicts $S(t)> 0$, $t\in[0,\tau)$. Hence, by contradiction we leave the upper left quadrant with $\tau<\tau_{\max,1d}$.\\
  %
  \emph{Case ii)}: $\epsilon(0) >0$:
  Note that \cref{eq:eps_geq_eps0_disturbed} and $\epsilon(0) >0$ imply $\epsilon(t) >0$ and hence $S(t)>0$ for all $t\in [0,\tau]$.\\
  \emph{Case ii) a)}: $\epsilon(0) >0$, $\cmin\geq 1$: Applying \cref{eq:bound_ab-term} from \cref{lma:bound_ab-term} to the first term of \cref{eq:dot_r_disturbed} and using \cref{eq:bound_norm_x} yields
  \begin{align}
    \dot{R}(t) & \geq   - \kcf\frac{\cmin}{\cmin^2+1} + \frac{\kcf}{\cmax} - \norm{\xv(0)} \zmax \nonumber                                     \\
               & = \kcf \frac{\cmin^2-\cmin\cmax+1}{\cmax\left(\cmin^2+1\right)} - \norm{\xv(0)} \zmax. \label{eq:bound_dotR_above_disturbed2}
  \end{align}
  Consider the disturbance bound
  \begin{equation} \label{eq:bound_z_above_disturbed2}
    \zmax \leq \kcf \frac{\cmin^2-\cmin\cmax+1}{2\norm{\xv(0)}\cmax\left(\cmin^2+1\right)},
  \end{equation}
  where the right hand side is positive due to \cref{eq:bound_cmax}.
  For contradiction, suppose that $\tau>\tau_{\max,2ad}$ with
  \begin{equation}
    \tau_{\max,2ad}  = \frac{- 2R(0)\cmax\left(\cmin^2+1\right)}{\kcf\left(\cmin^2-\cmin\cmax+1\right)}. \label{eq:tmax_above_disturbed2}
  \end{equation}
  Then, the integration of \cref{eq:bound_dotR_above_disturbed2} yields
  \begin{align*}
     & R(\tau_{\max,2ad}) \geq R(0) +                                                                                               \\
     & \left(\frac{\kcf\left(\cmin^2-\cmin\cmax+1\right)}{\cmax\left(\cmin^2+1\right)} - \norm{\xv(0)} \zmax \right)\tau_{\max,2ad} \\
     & \geq R(0) + \frac{\kcf\left(\cmin^2-\cmin\cmax+1\right)}{2\cmax\left(\cmin^2+1\right)}\tau_{\max,2ad} = 0,
  \end{align*}
  which contradicts $R(t)< 0$, $t\in[0,\tau)$. Hence, by contradiction we leave the upper left quadrant with $\tau\leq\tau_{\max,2ad}$.\\
  %
  \emph{Case ii) b)}: $\epsilon(0) >0$, $\cmin<1$: Analogously, applying \cref{eq:bound_ab-term} from \cref{lma:bound_ab-term} to the first term of \cref{eq:dot_r_disturbed} yields
  \begin{align}
    \dot{R}(t) & \geq   - \frac{\kcf}{2} + \frac{\kcf}{\cmax} - \norm{\xv(0)} \zmax \nonumber                   \\
               & \geq \kcf \frac{2-\cmax}{2\cmax} - \norm{\xv(0)} \zmax. \label{eq:bound_dotR_above_disturbed3}
  \end{align}
  Consider the disturbance bound
  \begin{equation} \label{eq:bound_z_above_disturbed3}
    \zmax \leq \frac{\kcf \left(2-\cmax\right)}{4 \cmax \norm{\xv(0)}},
  \end{equation}
  which is positive due to \cref{eq:bound_cmax}.
  For contradiction, suppose that $\tau>\tau_{\max,2bd}$ with
  \begin{equation}
    \tau_{\max,2bd}  = \frac{- 4 R(0) \cmax}{\kcf\left(2-\cmax\right)}. \label{eq:tmax_above_disturbed3}
  \end{equation}
  Then, the integration of \cref{eq:bound_dotR_above_disturbed3} yields
  \begin{align*}
     & R(\tau_{\max,2bd})                                                                                     \\
     & \geq R(0) + \left(\frac{\kcf\left(2-\cmax\right)}{2\cmax} - \norm{\xv(0)} \zmax \right)\tau_{\max,2bd} \\
     & \geq R(0) + \frac{\kcf\left(2-\cmax\right)}{4\cmax}\tau_{\max,2bd}                                     \\
     & = R(0) + \frac{\kcf\left(2-\cmax\right)}{4\cmax} \frac{- 4 R(0) \cmax}{\kcf\left(2-\cmax\right)} = 0,
  \end{align*}
  which contradicts $R(t)< 0$, $t\in[0,\tau)$. Hence, by contradiction we leave the upper left quadrant with $\tau\leq\tau_{\max,2bd}$.\\
  %
  \textbf{Part II.}
  In order to show that $|\epsilon(t)|\geq \frac{|\epsilon(0)|}{2}$, we first need to establish an upper bound for $|\dot{\epsilon}|$, for which we use \cref{eq:dot_r_disturbed,eq:dot_s_disturbed,eq:dot_c} within \cref{eq:dot_eps_disturbed} to get
  \begin{align} \label{eq:def_dot_eps}
    \begin{split}
      \frac{\dot{\epsilon}}{\kcf}  = & \frac{c R S - R^2}{R^2+S^2}+c \frac{\norm{\dot{\xv}}^2}{\kcf}+\frac{\left(\xv \times \zv \right) \cdot \bv}{\kcf} + \frac{c \xv \cdot \zv}{\kcf} \\
      & - 2c \frac{\dot{\xv} \cdot \zv}{\kcf\norm{\dot{\xv}}^2} R.
    \end{split}
  \end{align}
  To simplify this expression, we note that the following conditions hold
  \begin{align*}
    -R=                                                             & |\xv \cdot \dot{\xv}|                             \\
    \frac{c R S - R^2}{R^2+S^2}+c \frac{\norm{\dot{\xv}}^2}{\kcf} = & \frac{S}{\norm{\xv}^2\norm{\dot{\xv}}^2} \epsilon
  \end{align*}
  \begin{align*}
    -\norm{\xv(0)} \zmax           & \leq \left(\xv \times \zv \right) \cdot \bv                                  & \leq & \, \norm{\xv(0)} \zmax           \\
    -\cmax \norm{\xv(0)} \zmax     & \leq c \xv \cdot \zv                                                         & \leq & \, \cmax \norm{\xv(0)} \zmax     \\
    - 2 \cmax  \norm{\xv(0)} \zmax & \leq 2c \frac{\dot{\xv} \cdot \zv}{\norm{\dot{\xv}}^2} |\xv \cdot \dot{\xv}| & \leq & \, 2 \cmax  \norm{\xv(0)} \zmax,
  \end{align*}
  with $R<0$, $\norm{\xv} \leq \norm{\xv(0)}$ (cf. \cref{eq:bound_norm_x}) and $\norm{\zv} \leq \zmax$.\\
  %
  \emph{Case i)}: $\epsilon(0) < 0$:
  From Part I we know that $S(t)\geq 0$ for all $t\in[0,\tau]$. In the following, we assume $\epsilon(t)<0$ for all $t\in[0,\tau]$, which will be recursively established at the end. Then, \cref{eq:def_dot_eps} yields
  \begin{align}
    \frac{\dot{\epsilon}}{\kcf} & \leq  \frac{S}{\norm{\xv(0)}^2\dxmax^2} \epsilon + \frac{\norm{\xv(0)}}{\kcf} \left(1 + 3\cmax \right)\zmax \nonumber \\
                                & \leq \frac{\norm{\xv(0)}}{\kcf} \left(1 + 3\cmax \right)\zmax. \label{eq:bound_dotEps_below}
  \end{align}
  Consider $\tau \leq \tau_{\max,1d}$ with $\tau_{\max,1d}$ from \cref{eq:tmax_below_disturbed} and that the disturbance bound satisfies
  \begin{align}
    \zmax & \leq -\frac{\epsilon(0)}{2 \norm{\xv(0)} \left(1 + 3\cmax \right)}\frac{1}{\tau_{\max,1d}} \nonumber                                                   \\
          & \leq -\frac{\kcf \epsilon(0)}{4 S(0) \norm{\xv(0)} \left(1 + 3\cmax \right) \left(1+\cmax^2\right)}.\label{eq:bound_z_critical_dist_below_unintuitive}
  \end{align}
  Then, using \cref{eq:tmax_below_disturbed} and integrating \cref{eq:bound_dotEps_below} yields
  \begin{align}
    \epsilon(t) & \leq \epsilon(0) + \norm{\xv(0)} \left(1 + 3\cmax \right)\zmax t \nonumber                                          \\
                & \leq \epsilon(0) + \norm{\xv(0)} \left(1 + 3\cmax \right)\zmax \frac{2 S(0) \left(1+\cmax^2\right)}{\kcf} \nonumber \\
                & \leq \epsilon(0) + \frac{2 S(0) \norm{\xv(0)} \left(1 + 3\cmax \right)\left(1+\cmax^2\right)}{\kcf} \zmax \nonumber \\
                & \leq \frac{\epsilon(0)}{2} \label{eq:eps_leq_eps0_below}
  \end{align}
  for all $t\in[0,\tau]$.
  Note that Inequality \eqref{eq:bound_dotEps_below} used $\epsilon(t)<0$, which holds recursively using $\epsilon(0)<0$ and \cref{eq:eps_leq_eps0_below}.
  \\
  %
  \emph{Case ii)}: $\epsilon(0) > 0$:
  Analogously, we use $\epsilon(t)>0$ for $t\in[0,\tau]$, which will be recursively established in the following. Then, \cref{eq:def_dot_eps} yields
  \begin{align}
    \frac{\dot{\epsilon}}{\kcf} & \geq  \frac{S}{\norm{\xv(0)}^2\dxmax^2} \epsilon - \frac{\norm{\xv(0)}}{\kcf} \left(1 + 3\cmax \right)\zmax \nonumber \\
                                & \geq -\frac{\norm{\xv(0)}}{\kcf} \left(1 + 3\cmax \right)\zmax. \label{eq:bound_dotEps_above}
  \end{align}
  %
  \emph{Case ii) a)}: $\epsilon(0) > 0$, $\cmin\geq1$:
  Consider $\tau \leq \tau_{\max,2ad}$ with $\tau_{\max,2ad}$ from \cref{eq:tmax_above_disturbed2} and that the disturbance bound satisfies
  \begin{align}
    \zmax & \leq \frac{\epsilon(0)}{2 \norm{\xv(0)} \left(1 + 3\cmax \right)} \frac{1}{\tau_{\max,2ad}}  \nonumber                                                                                         \\
          & = -\frac{\kcf\epsilon(0)\left(\cmin^2 -\cmin\cmax + 1\right)}{4 R(0) \norm{\xv(0)} \cmax\left(1 + \cmin^2 \right)\left(1 + 3\cmax \right)}.\label{eq:bound_z_critical_dist_above1_unintuitive}
  \end{align}
  Then, using \cref{eq:tmax_above_disturbed2} and integrating \cref{eq:bound_dotEps_above} yields
  \begin{align}
    \epsilon(t) & \geq \epsilon(0) - \norm{\xv(0)} \left(1 + 3\cmax \right)\zmax t \nonumber               \\
                & \geq \epsilon(0) - \norm{\xv(0)} \left(1 + 3\cmax \right)\zmax \tau_{\max,2ad} \nonumber \\
                & \geq \frac{\epsilon(0)}{2} \label{eq:eps_geq_eps0_above1}
  \end{align}
  for all $t\in[0,\tau]$.\\
  %
  \emph{Case ii) b)}: $\epsilon(0) > 0$, $\cmin<1$:
  Consider $\tau \leq \tau_{\max,2bd}$ with $\tau_{\max,2bd}$ from \cref{eq:tmax_above_disturbed3} and that the disturbance bound satisfies
  \begin{align}
    \zmax & \leq \frac{\epsilon(0)}{2 \norm{\xv(0)} \left(1 + 3\cmax \right)} \frac{1}{\tau_{\max,2bd}}   \nonumber                                                 \\
          & = -\frac{\kcf\epsilon(0)\left(2 - \cmax \right)}{8 R(0)\norm{\xv(0)} \cmax\left(1 + 3\cmax \right)}.\label{eq:bound_z_critical_dist_above2_unintuitive}
  \end{align}
  Then, using \cref{eq:tmax_above_disturbed3} and integrating \cref{eq:bound_dotEps_above} yields
  \begin{align}
    \epsilon(t) & \geq \epsilon(0) - \norm{\xv(0)} \left(1 + 3\cmax \right)\zmax t \nonumber               \\
                & \geq \epsilon(0) - \norm{\xv(0)} \left(1 + 3\cmax \right)\zmax \tau_{\max,2bd} \nonumber \\
                & \geq \frac{\epsilon(0)}{2} \label{eq:eps_geq_eps0_above2}
  \end{align}
  for all $t\in[0,\tau]$.
  Note that Inequality \eqref{eq:bound_dotEps_above} used $\epsilon(t)>0$, which holds recursively using $\epsilon(0)>0$ and \cref{eq:eps_geq_eps0_above1,eq:eps_geq_eps0_above2}.\\
  %
  \textbf{Part III.}
  In the following, we use the derived bound $|\epsilon(t)|\geq \frac{|\epsilon(0)|}{2}, t\in[0,\tau]$ from \cref{eq:eps_leq_eps0_below,eq:eps_geq_eps0_above1,eq:eps_geq_eps0_above2} to establish a lower bound on the obstacle distance $\norm{\xv(t)}$.\\
  %
  \emph{Case i)}: $\epsilon(0) < 0$:
  From \cref{eq:def_epsilon,eq:eps_leq_eps0_below} and $S(t)>0$, it follows that for all $t \in [0,\tau]$
  \begin{equation*}
    R(t) \leq \frac{\frac{\epsilon(0)}{2}-S(t)}{c(t)} \leq  \frac{\epsilon(0)}{2\cmax} \Rightarrow  R(t)^2 \geq \frac{\epsilon(0)^2}{4\cmax^2}
  \end{equation*}
  and therefore $V_B(t) \overset{\eqref{eq:barrier_ab}}{\leq} \frac{\dxmax^2}{R(t)^2} \leq \frac{4 \dxmax^2 \cmax^2}{\epsilon(0)^2}$, which implies
  \begin{equation}
    \norm{\xv(t)}\geq \frac{|\epsilon(0)|}{2 \dxmax \cmax} 
  \end{equation}
  for all $t \in [0, \tau]$.
  Recalling that $S(0) \leq \norm{\xv(0)} \dxmax$ (cf. \cref{eq:def_s,eq:bound_norm_x}), we can replace the disturbance bound from \cref{eq:bound_z_critical_dist_below_unintuitive} with the following more restrictive, albeit more intuitive upper bound
  \begin{align}
    \zmax & \leq \frac{\kcf\left|\epsilon(0)\right|}{4 \norm{\xv(0)}^2 \dxmax \left(1 + 3\cmax \right) \left(1+\cmax^2\right)} \label{eq:bound_z_critical_dist_below} \\
          & \leq \frac{\kcf\left|\epsilon(0)\right|}{4 S(0) \norm{\xv(0)} \left(1 + 3\cmax \right) \left(1+\cmax^2\right)}. \nonumber
  \end{align}
  Consequently, all guarantees hold for $\norm{\zv} \leq \zmax = \min \left(\frac{\kcf}{2 \norm{\xv(0)} \left(1+\cmax^2\right)}, \frac{\kcf\left|\epsilon(0)\right|}{4 \norm{\xv(0)}^2 \dxmax \left(1 + 3\cmax \right) \left(1+\cmax^2\right)}\right)$.\\
  %
  \emph{Case ii)}: $\epsilon(0) > 0$:
  From \cref{eq:def_epsilon,eq:eps_geq_eps0_above1,eq:eps_geq_eps0_above2} and $R<0$, it follows that
  \begin{align*}
    S(t) & \geq \frac{\epsilon(0)}{2}-c(t)R(t)
    \geq \frac{\epsilon(0)}{2}
  \end{align*}
  and therefore $V_B \overset{\eqref{eq:barrier_ab}}{\leq} \frac{\dxmax^2}{S(t)^2} \leq \frac{4 \dxmax^2}{\epsilon(0)^2}$, which implies
  \begin{equation}
    \norm{\xv(t)}\geq \frac{\epsilon(0)}{2 \dxmax}
  \end{equation}
  for all $t \in [0, \tau]$.
  Recalling that $R(0) \leq \norm{\xv(0)} \dxmax$ (cf. \cref{eq:def_r,eq:bound_norm_x}), we can replace the disturbance bounds from \cref{eq:bound_z_critical_dist_above1_unintuitive,eq:bound_z_critical_dist_above2_unintuitive} with the following more restrictive, albeit more intuitive upper bounds
  \begin{align}
    \zmax & \leq \frac{\kcf|\epsilon(0)|\left(\cmin^2 -\cmin\cmax + 1\right)}{4 \norm{\xv(0)}^2 \dxmax \cmax\left(1 + \cmin^2 \right)\left(1 + 3\cmax \right)} \label{eq:bound_z_critical_dist_above1} \\
          & \leq \frac{\kcf|\epsilon(0)|\left(\cmin^2 -\cmin\cmax + 1\right)}{4 |R(0)| \norm{\xv(0)} \cmax\left(1 + \cmin^2 \right)\left(1 + 3\cmax \right)} \nonumber                                 \\
    \zmax & \leq \frac{\kcf|\epsilon(0)|\left(2 - \cmax \right)}{8 \norm{\xv(0)}^2 \dxmax \cmax \left(1 + 3\cmax \right)} \label{eq:bound_z_critical_dist_above2}                                      \\
          & \leq \frac{\kcf|\epsilon(0)|\left(2 - \cmax \right)}{8 |R(0)|\norm{\xv(0)} \cmax \left(1 + 3\cmax \right)}. \nonumber
  \end{align}
  Consequently, all guarantees hold for $\zv \leq \zmax = \min \left(\frac{\kcf\left(\cmin^2-\cmin\cmax+1\right)}{2\norm{\xv(0)}\cmax\left(\cmin^2+1\right)}, \frac{\kcf|\epsilon(0)|\left(\cmin^2 -\cmin\cmax + 1\right)}{4 \norm{\xv(0)}^2 \dxmax \cmax\left(1 + \cmin^2 \right)\left(1 + 3\cmax \right)} \right)$ if $\cmin \geq 1$ or\\
  $\zv \leq \zmax = \min \left(\frac{\kcf \left(2-\cmax\right)}{4 \norm{\xv(0)} \cmax}, \frac{\kcf|\epsilon(0)|\left(2 - \cmax \right)}{8 \norm{\xv(0)}^2 \dxmax \cmax \left(1 + 3\cmax \right)} \right)$ if $\cmin < 1$.
\end{proof}
%
%
\bibliographystyle{IEEEtran}
\bibliography{IEEEabrv,bibProofCA}

\vspace*{-2em}
\begin{IEEEbiography}[{\includegraphics[width=1in,height=1.25in,clip,keepaspectratio]{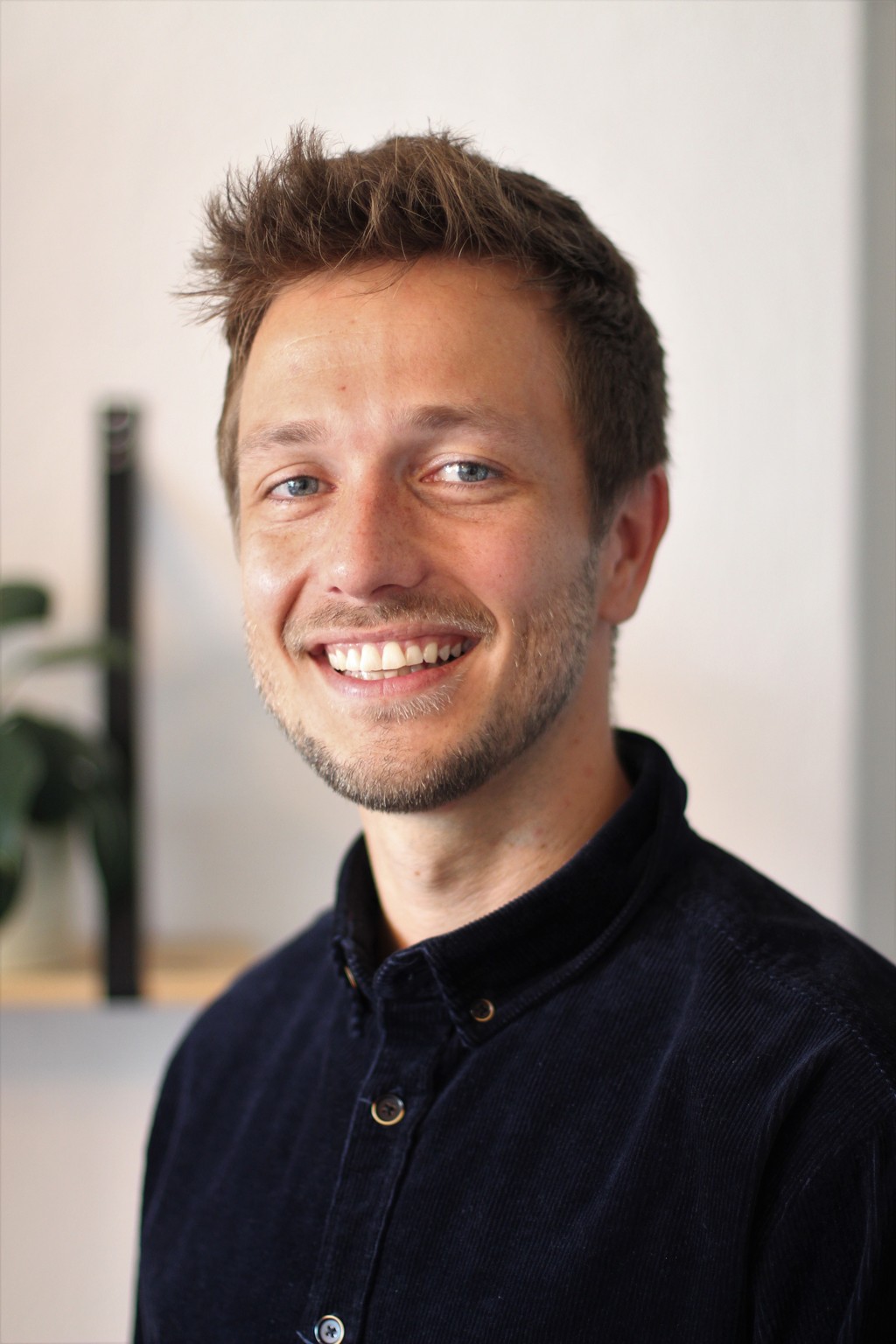}}]{Marvin Becker}
    received his Master degree in mechanical engineering from the Technical University Munich, Germany, in 2016.
    He is currently working toward the Ph.D. degree in electrical engineering at the Institute of Automatic Control (IRT) at Leibniz University Hannover.
    His current research interests are in the area of motion planning and collision avoidance for robotic manipulators. \end{IEEEbiography}
\vspace*{-2em}
\begin{IEEEbiography}[{\includegraphics[width=1in,height=1.25in,clip,keepaspectratio]{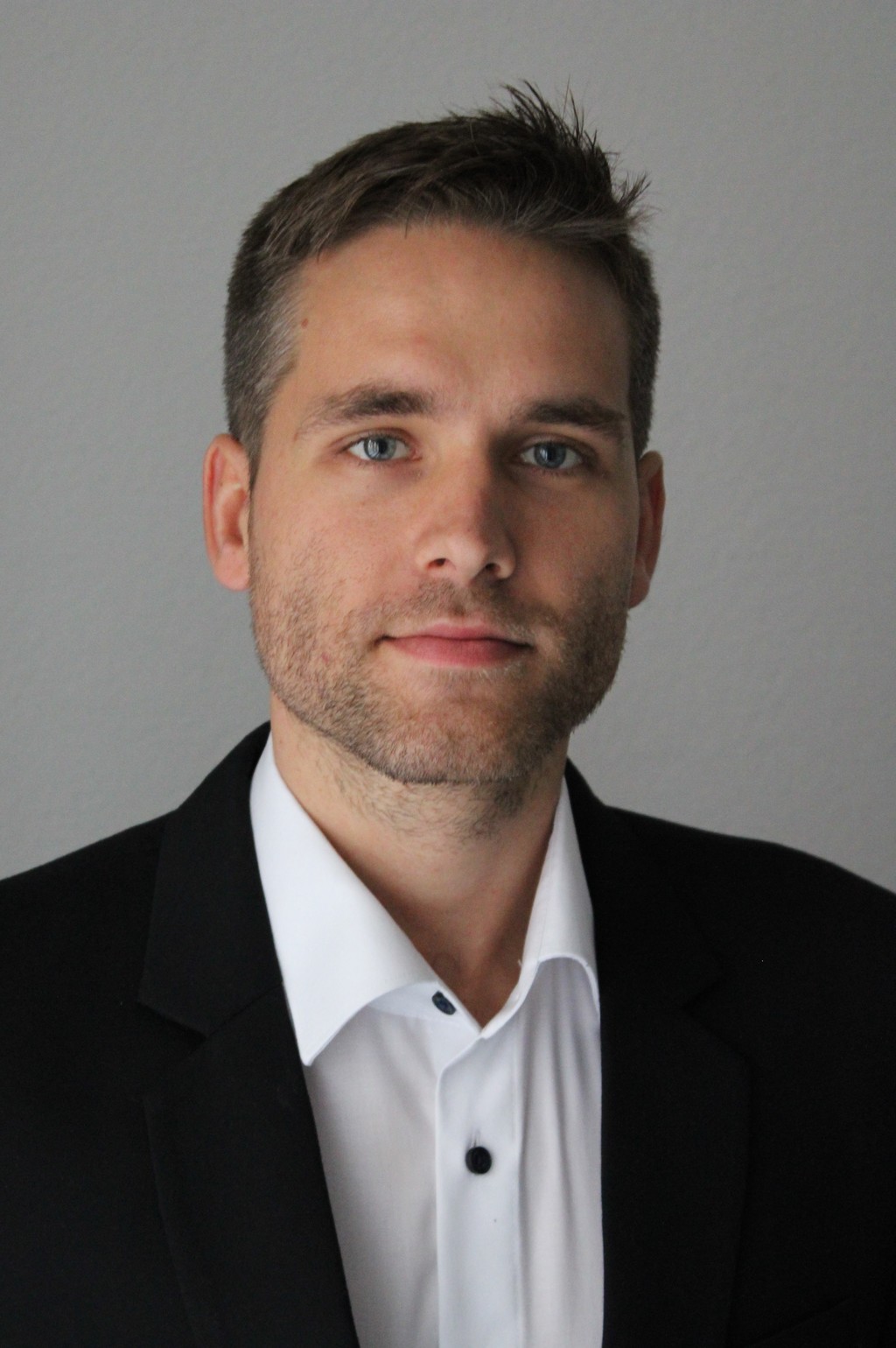}}]{Johannes K\"ohler}
    received his Master degree in Engineering Cybernetics from the University of Stuttgart, Germany, in 2017.
    In 2021, he obtained a Ph.D. in mechanical engineering, also from the University of Stuttgart,
    Germany, for which he received the 2021 European Systems \& Control Ph.D. award.
    He is currently a postdoctoral researcher at the Institute for Dynamic Systems and Control (IDSC) at ETH Zürich.
    His current research interests are in the area of model predictive control and control and estimation for nonlinear uncertain systems.
\end{IEEEbiography}
\vspace*{-2em}
\begin{IEEEbiography}[{\includegraphics[width=1in,height=1.25in,clip,keepaspectratio]{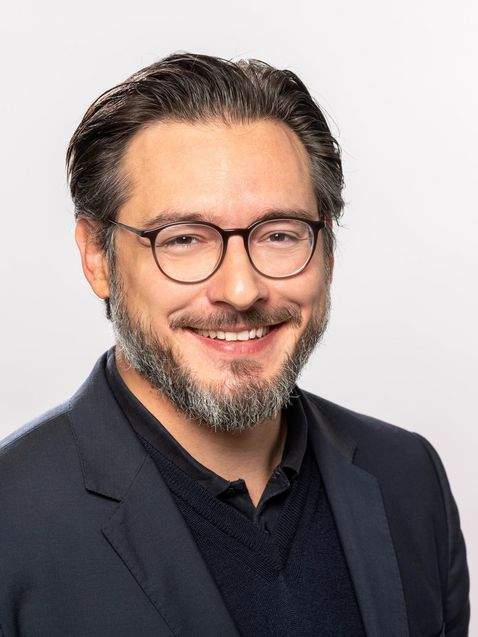}}]{Sami Haddadin}
    received the Dipl.-Ing. degree in electrical engineering in 2005, the M.Sc. degree in computer science in 2009 from Technical University of Munich (TUM), Munich, Germany, the Honours degree in technology management in 2007 from Ludwig Maximilian University, Munich, Germany, and TUM, and the Ph.D. degree in safety
    in robotics from RWTH Aachen University, Aachen, Germany, in 2011.
    He is currently a full professor and chair of Robotics and Systems Intelligence at the Technical University of Munich (TUM) and the founding director of the Munich Institute of Robotics and Machine Intelligence (MIRMI).
    He has received numerous awards for his scientific work, including the George Giralt Ph.D. Award (2012), the RSS Early Career Spotlight (2015), the IEEE/RAS Early Career Award (2015), the Alfried Krupp Award for Young Professors (2015), the German President's Award for Innovation in Science and Technology (2017) and the Leibniz Prize (2019).
    His research interests include physical human-robot interaction, nonlinear robot control, real-time motion planning, real-time task and reflex planning,
    robot learning, optimal control, human motor control, variable impedance actuation, and safety in robotics.
\end{IEEEbiography}
\vspace*{-2em}
\begin{IEEEbiography}[{\includegraphics[width=1in,height=1.25in,clip,keepaspectratio]{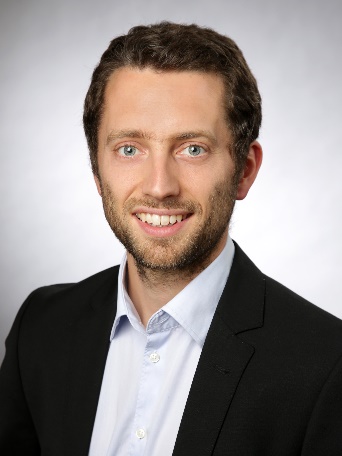}}]{Matthias A. Müller}
    received a Diploma degree in Engineering Cybernetics from the University of Stuttgart, Germany, an M.Sc. in Electrical and Computer Engineering from the University of Illinois at Urbana-Champaign, US (both in 2009), and a Ph.D. from the University of Stuttgart in 2014. Since 2019, he is director of the Institute of Automatic Control and full professor at the Leibniz University Hannover, Germany. His research interests include nonlinear control and estimation, model predictive control, and data- and learning-based control, with application in different fields including biomedical engineering and robotics. He has received various awards for his work, including the 2015 EECI PhD award, the inaugural Brockett-Willems Outstanding Paper Award for the best paper published in Systems \& Control Letters in the period 2014-2018, an ERC starting grant in 2020, and the IEEE CSS George S. Axelby Outstanding Paper Award 2022. He serves as an editor of the International Journal of Robust and Nonlinear Control and as a member of the Conference Editorial Board of the IEEE Control Systems Society.
\end{IEEEbiography}

\end{document}